\documentclass[%
 amsmath,amssymb,
 aps,
floatfix,
]{revtex4-2}

\usepackage{mathtools}
\usepackage{amssymb}
\usepackage{graphicx}
\usepackage{etoolbox}
\usepackage[utf8]{inputenc}
\usepackage[OT1]{fontenc}
\usepackage{amsthm}

\newcommand{\dR}{\mathbb{R}}

\newcommand{\dE}{\mathbb{E}}

\newcommand{\dP}{\mathbb{P}}

\newcommand{\cC}{\mathcal{C}}

\newcommand{\cN}{\mathcal{N}}

\newcommand{\cR}{\mathcal{R}}
\newcommand{\cS}{\mathcal{S}}

\DeclarePairedDelimiterXPP{\Pb}[1]{\mathbb{P}}{\lparen}{\rparen}{}{ #1}
\DeclarePairedDelimiterXPP{\E}[1]{\mathbb{E}}[]{}{ #1}
\DeclarePairedDelimiterX{\Set}[1]\lbrace\rbrace{ #1}

\DeclareMathOperator{\tr}{tr}

\DeclareMathOperator*{\argmin}{arg\,min}
\DeclareMathOperator{\prox}{prox}
\DeclarePairedDelimiterX{\norm}[1]\lVert\rVert{\ifblank{#1}{\: \cdot \:}{#1}}

\newcommand{\bilingualcommand}[3]{%
	\newcommand{#1}[1][\ ]{%
		##1%
		\iflanguage{english}{\text{#2}}{%
			\iflanguage{french}{\text{#3}}{}%
		}%
		##1%
	}%
}

\bilingualcommand{\where}{where}{où}
\bilingualcommand{\textif}{if}{si}
\bilingualcommand{\textand}{and}{et}
\bilingualcommand{\textiff}{if and only if}{si et seulement si}
\bilingualcommand{\otherwise}{otherwise}{sinon}

\usepackage{dcolumn}
\usepackage[utf8]{inputenc} 
\usepackage[T1]{fontenc}    
\usepackage[colorlinks=true, linkcolor=blue,citecolor=blue]{hyperref}       
\usepackage{url}            
\usepackage{booktabs}       
\usepackage{amsmath, amssymb, amsthm, amsfonts}       
\usepackage{nicefrac, xfrac}       
\usepackage{microtype}      
\usepackage{xcolor}         
\usepackage{bm}
\usepackage{enumitem}
\usepackage{graphicx}
\usepackage{algorithm}
\usepackage{algorithmic}
\usepackage{hyperref}

\makeatletter
\newtheorem*{rep@theorem}{\rep@title}
\newcommand{\newreptheorem}[2]{%
\newenvironment{rep#1}[1]{%
 \def\rep@title{#2 \ref{##1}}%
 \begin{rep@theorem}}%
 {\end{rep@theorem}}}
\makeatother

\newtheorem{theorem}{Theorem}
\newreptheorem{theorem}{Theorem}

\newtheorem{lemma}{Lemma}

\begin{document}

\preprint{APS/123-QED}

\title{Gaussian Universality of Perceptrons with Random Labels}

\author{Federica Gerace$^{1,4}$}
\author{Florent Krzakala$^2$}%
\author{Bruno Loureiro$^{2,3}$}
\author{Ludovic Stephan$^2$}
\author{Lenka Zdeborov\'a$^4$}
\affiliation{
$^1$ International School of Advanced Studies (SISSA). Trieste, Italy.
}%
\affiliation{%
 $^2$EPFL, Information, Learning and Physics (IdePHICS) lab., Lausanne,  Switzerland 
}
\affiliation{%
 $^3$D\'epartement d'Informatique, \'Ecole Normale Sup\'erieure (ENS) - PSL \& CNRS,  F-75230 Paris cedex 05, France}
\affiliation{%
$^4$EPFL Statistical Physics of Computation (SPOC) lab.,Lausanne,  Switzerland 
}

\date{\today}

\begin{abstract}
While classical in many theoretical settings --- and in particular in statistical physics-inspired works--- the assumption of Gaussian {\it i.i.d.} input data is often perceived as a strong limitation in the context of statistics and machine learning. In this study, we redeem this line of work in the case of generalized linear classification,  a.k.a. the perceptron model, with random labels. 
We argue that there is a large universality class of high-dimensional input data for which we obtain the same minimum training loss as for Gaussian data with corresponding data covariance. In the limit of vanishing regularization, we further demonstrate that the training loss is independent of the data covariance. On the theoretical side, we prove this universality for an arbitrary mixture of homogeneous Gaussian clouds.
Empirically, we show that the universality holds also for a broad range of real datasets. 
\end{abstract}

\maketitle


\section{Introduction}
\label{Introduction}
Statistical physics studies of artificial neural networks have a long history, including many works that continue to have an impact on the current investigations of deep neural networks. A large fraction of this continuing line of works has focused on Gaussian input data, see \cite{gardner1989three,krauth1989storage,seung1992statistical} for some of the earliest and most influential examples. However, 
the Gaussian data assumption is not limited to works from statistical physics of learning. Indeed, it is a widespread assumption in the high-dimensional statistics literature, where it is also known under the umbrella of \emph{Gaussian design}, see for example \cite{donoho2009observed,candes2020phase, Bartlett30063}. 
Despite being both common and convenient for doing theory, {\it i.i.d.} Gaussian data might come across as a stringent limitation at first sight, out-of-touch with the real-world practice where data is structured. Indeed, an important branch of statistical learning theory is data-agnostic and avoids making too specific assumptions on the data distribution \cite{shalev2014understanding}.
However, a number of recent observations (both heuristic and rigorous) suggest that the Gaussian assumption is not always that far-stretched for high-dimensional data (see for instance \cite{goldt2019modelling,seddik_2020_random,bordelon2020spectrum,hu2020universality,loureiro2021learning} and references therein). The goal of the present work is to redeem the Gaussian hypothesis for perhaps the simplest, yet deeply fundamental, problem of high-dimensional statistics: the perceptron problem, a.k.a. generalized linear classification, with random labels.

Models with random labels are ubiquitous in the theory of machine learning. The problem of how many randomly labelled Gaussian patterns a perceptron model can fit, known as the \emph{storage capacity problem}, is at the root of the historical interest of the statistical physics community for machine learning problems. Indeed, works on this classical subject span more than four decades \cite{gardner1989three,krauth1989storage,brunel1992information,franz2017universality,ding2019capacity,aubin2019storage,abbaras2020rademacher,montanari2021tractability,alaoui2020algorithmic}. The interest for random labels is also not bound to the statistical physics of learning community. They appear in several contexts in statistical learning theory, such as in the definition of Rademacher complexities \cite{shalev2014understanding,vapnik1999nature}, in Wendel/Cover's pioneering studies \cite{wendel1962problem,cover1965geometrical} or in thought-provoking numerical experiments with deep learning \cite{zhang2021understanding,maennel2020neural},  

In this work, we ask: how would these theories for random labels change if using a realistic data set instead of a Gaussian one? We consider the training loss of generalized linear classifiers (perceptrons) trained on random labels, including ridge, hinge and logistic classification \cite{james2013introduction}, but also kernel methods \cite{steinwart2008support} and neural networks trained in the lazy regime \cite{chizat2018lazy} (the so-called neural tangent kernel \cite{jacot2018neural}), as well as with engineered features such as the scattering transform \cite{bruna2013invariant}. We focus on the thermodynamic limit (known as the high-dimensional setting in statistics) where both $n$ (the number of training samples) and $p$ (the input dimension) go to infinity at a fixed rate $\alpha=n/p$. 

Our main result is to argue that in the aforementioned setting with random labels many input data distributions actually have the same learning properties as Gaussian data, thus providing a rather surprising Gaussian universality result for this problem. In particular, the minimum training loss for a wide range of settings is the same as that of a corresponding Gaussian problem with matching data covariance. Furthermore, in the limit of vanishing regularization, we show that Gaussian universality is even stronger, as the minimum training loss is independent of the data covariance (and therefore the same as the one of \emph{i.i.d.} Gaussian data). In other words: as far as random labels are concerned, it turns out that the theoretical results derived under the Gaussian data assumptions capture what is actually happening in practice. Certainly, the value of the interpolation (or capacity) threshold was known to be universal and occurs (for full-rank data) at $n=p$ for ridge regression, and for $n=2p$ for linear classifiers (perceptrons) \cite{cover1965geometrical}; however, the fact that the loss itself is universal is a stronger statement that redeems an entire line of work using the Gaussian data assumption, and in particular a large part of those from statistical physics of learning.

\vspace{-5mm}

\subsection*{Summary of main results}
\vspace{-2mm}
The main points of the present work can be summarized by Figures \ref{fig:all_datasets_zero_reg} and \ref{fig:all_datasets_finite_reg}, which show the training loss of real-world data sets trained with random labels and various feature maps, compared with the (analytical) prediction derived for Gaussian data with matching covariance. The code used to run these experiments is publicly available in a \href{https://github.com/IdePHICS/RandomLabelsUniversality}{GitHub repository}. As illustrated in these plots, Gaussian universality  seems to hold even for finite-dimensional data, and for actual real datasets. Notably, we observe that when using random labels the training losses plotted as a function of the ratio between the number of samples and the dimension $\alpha=n/p$ are indistinguishable from results obtained for Gaussian input data when using MNIST \cite{lecun_1998_gradient}, fashion-MNIST \cite{xiao2017/online}, CIFAR10 \cite{krizhevsky_2009_learning} preprocessed through various standard feature maps. This conclusion seems robust and holds for different features of the raw data, such as random features \cite{rahimi2007random} or the convolutional scattering transform \cite{bruna2013invariant,andreux2020kymatio}. It also holds, as we prove, if we simply use a synthetic Gaussian mixture model, a classical model for complex multi-model data. The agreement between the real world and the asymptotic Gaussian theory is striking. While we may expect that such data could be approximated by a multimodal distribution such as a Gaussian mixture with enough modes, it should come as a rather puzzling fact that they lead to the same loss as a single Gaussian cloud. Our main contribution is to provide a rigorous theoretical foundation for these observations, that vindicates the classical line of works on Gaussian design, in particular the one stemming from statistical physics. 

We list here our \textbf{main results}:
\begin{figure}[t!]
\begin{center}
\centerline{\includegraphics[width=500pt]{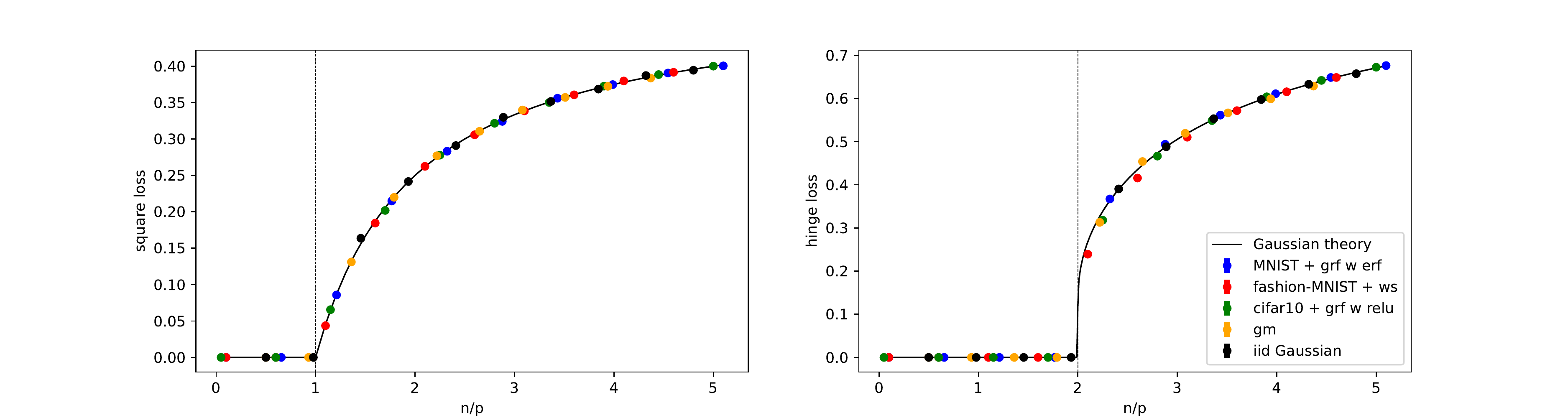}}
\caption{Training loss as function of the number of samples $n$ per input dimension $p$ at regularization $\lambda = 10^{-15}$. In the left panel the square loss, and in the right panel the hinge loss. The black solid line represents the outcome of the replica calculation for \emph{i.i.d} Gaussian inputs, namely when the covariance matrix $\Sigma$ corresponds to the identity matrix. Dots refer to numerical simulations on different full-rank datasets. In particular, blue dots correspond to MNIST with Gaussian random features and error function non-linearity, red dots correspond to fashion-MNIST with wavelet scattering transform, green dots correspond to CIFAR10 in grayscale with Gaussian random features and ReLU non-linearity, yellow dots corresponds to a mixture of Gaussians, with means $\bm{\mu}_{\pm} = \left( \pm 1, 0,...,0 \right)$, covariances $\Sigma_{\pm}$ both equal to the identity matrix and relative class proportions $\rho_{\pm} = \sfrac{1}{2}$. Finally, black dots correspond to \emph{i.i.d.} Gaussian inputs.
}
\label{fig:all_datasets_zero_reg}
\end{center}
\vskip -0.2in
\end{figure}
\begin{figure}[t!]
\begin{center}
\centerline{\includegraphics[width=500pt]{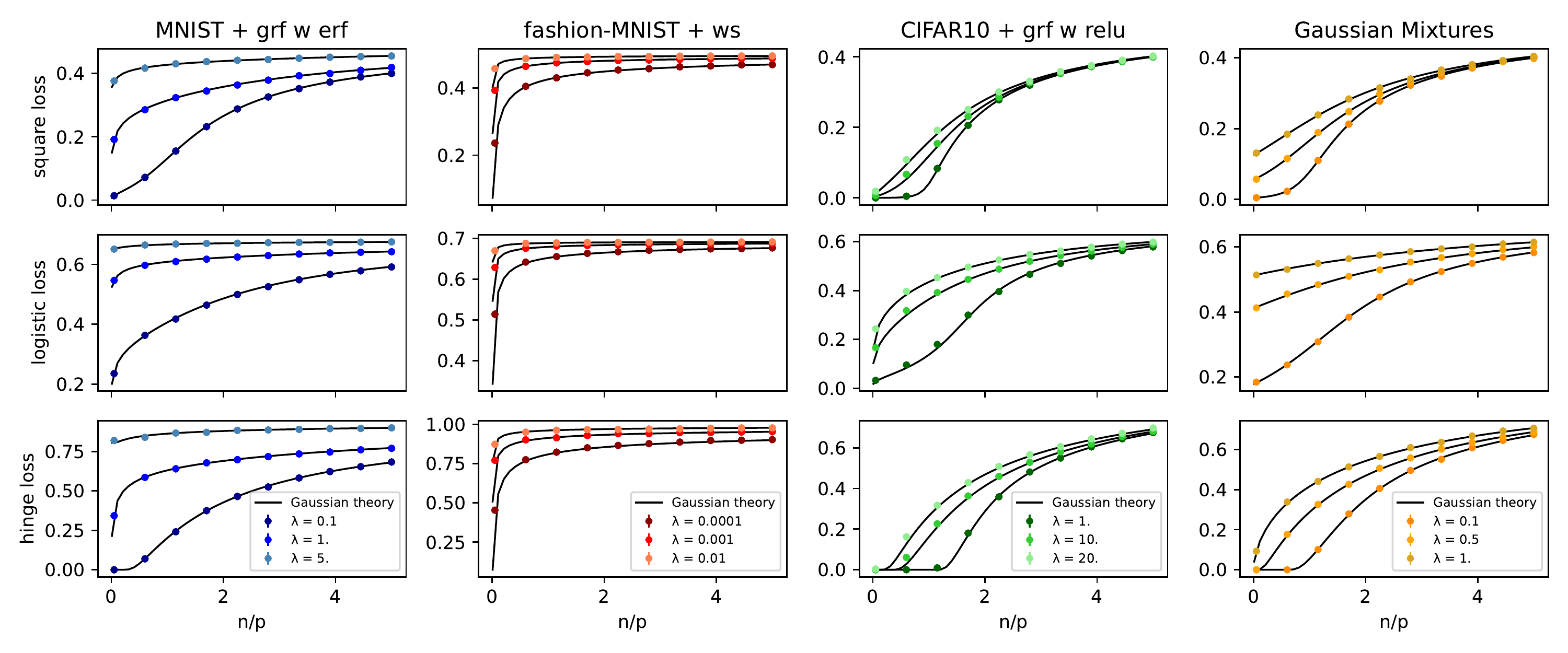}}
\caption{This figure shows the training loss as a function of the number of samples $n$ per dimension $p$ at finite regularization $\lambda$. In the top panel the square loss, and in the bottom panel the hinge loss. The first column refers to MNIST with Gaussian random features and error function non-linearity, the second column corresponds to fashion-MNIST with wavelet scattering transform, the third column corresponds to CIFAR10 in grayscale with Gaussian random features and ReLU non-linearity, the fourth column corresponds to a mixture of Gaussians, with means $\bm{\mu}_{\pm} = \left( \pm 1, 0,...,0 \right)$, covariances $\Sigma_{\pm}$ both equal to the identity matrix and relative class proportions $\rho_{\pm} = \sfrac{1}{2}$. Black solid lines correspond to the outcome of the replica calculation, obtained by assigning to $\Sigma$ the covariance matrix of each dataset plus the corresponding transformation.  The coloured dots correspond to the simulations for different values of $\lambda$, as specified in the plot legend. Simulations are averaged over $10$ samples \& the error bars are not visible at the plot scale.}
\label{fig:all_datasets_finite_reg}
\end{center}
\vskip -0.4in
\end{figure}
\begin{itemize}
    \item[a)] We provide a strong universality theorem for linear interpolators corresponding to ridgeless regression (with vanishing regularization) in high-dimensions and random labels, Theorem \ref{thm:ols}. Informally, we prove that a perceptron trained on randomly labelled Gaussian mixture data (a setting that encompasses complex multi-modal distributions) has the same minimum asymptotic loss as a perceptron trained on randomly labelled Gaussian data with isotropic covariance, that is $\mathcal{E}_{\ell}(\alpha)=\sfrac{1}{2}(1-\sfrac{1}{\alpha})_{+}$. This provides a theoretical explanation for the phenomena illustrated in Fig. \ref{fig:all_datasets_zero_reg} left.
    \item[b)] Under an additional homogeneity assumption on the different modes of the data, Gaussian universality can be generalised to {\it any convex loss} (and we conjecture that it is valid for non-convex losses as well), Theorem \ref{theorem:cov}.  This provides a theoretical explanation of the phenomena illustrated in Fig. \ref{fig:all_datasets_zero_reg} (right).
    \item[c)] At finite regularization and under the same homogeneity assumption, we show that the asymptotic training loss depends solely on the {\it data covariance matrix}, such that it is, again, the same loss as the one of a single Gaussian cloud with matching covariance, Theorem \ref{thm:Gauss_universality}. This is illustrated in Fig. \ref{fig:all_datasets_finite_reg}.
\end{itemize}

The proof technique used to establish these universality theorems has an interest on its own. It builds on recent progress in high-dimensional statistics and in mathematical insights drawn from the replica method in statistical physics. In particular, we provide an {\it explicit} matching of the expression (obtained from a rigorous proof of the replica prediction) for the asymptotic minimal loss \cite{thrampoulidis2018precise, montanari2019generalization, loureiro2021learning, loureiro2021learning_gm}. We further demonstrate the strong universality for ridge regression with vanishing regularization, again by showing explicitly that the exact solution \cite{dobriban2018high,hastie2019surprises,loureiro2021learning} reduces to one of the homogeneous Gaussian cases. These results are obtained through methods that have been developed from statistical physics and mathematical physics-inspired techniques.

\subsection*{Further Related work}

\paragraph*{The perceptron ---}
The question of how many samples can be perfectly fitted by a linear model is a classical one. For a ridge classifier, it amounts to asking whether a linear system of $n$ equations with $p$ unknowns is invertible so that for full-rank data the transition arises at $n=p$. For the 0/1 loss or its convex surrogate such as the hinge loss, the question of linear separability was famously discussed by \cite{cover1965geometrical} who showed that for full-rank data the transition is given by $n=2p$. In both cases, the transition is universal and does not depend on details of the data distribution (provided it is full rank, otherwise the rank replaces the dimension). For Gaussian data, such questions have witnessed a large amount of attention in the statistical physics community
~\citep{gardner1988optimal,gardner1989three,krauth1989storage,derrida1991finite,brunel1992information,franz2017universality} but also recently in theoretical computer science~\citep{ding2019capacity,aubin2019storage,abbaras2020rademacher,montanari2021tractability,alaoui2020algorithmic}. It is one of our goals to attract attention to these works, given the Gaussian universality we present shows that their relevance is not limited to idealistic Gaussian data.\looseness=-1

\paragraph*{Random Labels ---} Random labels are a fundamental and useful concept in machine learning. The pioneering work of \citep{zhang2021understanding}, for instance, was instrumental in the modern critics of classical measures of model complexity, including the Rademacher complexity or the VC-dimension. These considerations have driven an entire line of research aiming to find substantial differences between learning with true and random labels, for instance in training time \citep{arpit2017closer,han2018co,zhang2018generalized}, in minima sharpness \citep{keskar2016large,neyshabur2017exploring} or in what neural networks can actually learn with random labels \citep{maennel2020neural}. It has also been recently claimed \citep{maennel2020neural} that pre-training on random labels under a given initial condition scaling can consistently speed up neural network training on both true and random labels, with respect to training from scratch. 

\paragraph*{Gaussian Universality ---}
There has been much progress on a similar, though more restricted, Gaussian universality for random feature maps on Gaussian input data \cite{rahimi2007random}. Following  early insights by \cite{el2010spectrum}, the authors of \cite{pennington2017nonlinear, mei2019generalization} showed that the empirical distribution of the Gram matrix of random features is asymptotically equivalent to a linear model with matched covariance. This  was extended to generic convex losses by \cite{gerace2020generalisation} using the heuristic replica method, and proven in \cite{dhifallah2021phase}. A specific \emph{Gaussian Equivalence Principle} \cite{goldt2019modelling} for learning with random features has been proven in a succession of works for convex penalties in \cite{goldt2020gaussian,hu_universality_2021} and some non-convex ones in \cite{montanari2022universality}. Early ideas on Gaussian universality have also appeared in the context of signal processing and compressed sensing in \cite{doi:10.1098/rsta.2009.0152, 8006947, 5730603, NEURIPS2019_dffbb6ef, NIPS2017_136f9513}. These theoretical results, however, fall short when considering realistic datasets as we do in this work. Indeed, these previous works considered only uni-modal Gaussian data (observed through random feature maps), a situation far from realistic multi-modal, complex, real-world datasets.  Instead, \cite{seddik_2020_random,louart2018random,futureuniversality} argued that real datasets can be efficiently approximated in high dimensions by a finite {\it mixture of Gaussians}. These, of course, are multi-modal distributions that cannot be approximated by a single Gaussian. Gaussian mixtures will be the starting point of our theory.

Finally, we note that the observation that Gaussian data can fit or represent well some real data has been heuristically observed in many situations, but without theoretical justification and often limited to ridge regression, see e.g.  \cite{bordelon2020spectrum,li2021statistical,loureiro2021learning,cui2021generalization,Ingrosso2022}. 

\section{Setting, notation, and Asymptotic formulas}
\label{section3}
The focus of the present work is the analysis of high-dimensional binary linear classification (aka perceptron) on a dataset $\mathcal{D} = \left \{ \left(\bm{x}_{\mu}, y_{\mu}\right) \right \}_{\mu = 1}^n$. We shall consider a minimization problem of the form
\begin{equation}
\label{eq:def_min_problem_main}
 \widehat\cR_n^*(\bm X, \bm y) = \inf_{\bm \theta} \frac1n\sum_{\mu=1}^n \ell(\bm \theta^\top \bm x_\mu, y_\mu) + 
 \frac{\lambda}{2} \vert\vert \bm{\bm \theta} \vert\vert_2^2 ,
\end{equation}
where the $\bm x_\mu \in \dR^p$ are input vectors, $y_\mu\in \{-1,+1\}$ are binary labels.
We assume that the loss $\ell$ only depends on the inputs $\bm{x}_\mu$ through a one-dimensional projection $\bm \theta^\top \bm x_\mu$, and we work in the so-called \emph{thermodynamic} or \emph{proportional high-dimensional limit}, where $n, p$ go to infinity with
\[\frac np \to \alpha > 0. \]

In practice, practitioners seldom use the raw data {\bf x} directly in their linear classifiers and usually perform a preprocessing step. For instance, instead of using the raw MNIST, a classical approach is to use a fixed feature map, and to observe the data as ${\bf x}=\sigma(F {\bf x})$, with $F$ a random matrix. This is called the random feature map \cite{rahimi2007random}, and it has the advantage, among others, that the effective data ${\bf x}$ are full-rank. One may use more complicated such as the convolutional scattering transform \cite{brunel1992information,bruna2013invariant}, or even pre-trained neural networks, a situation called transfer learning \cite{torrey2010transfer,gerace2022probing}. We shall as well applied such transforms to our real data, in order to avoid theoretical pitfalls in direct space (in images some pixels are always zero for instance, so that the data may not even be full-rank). There is also one more advantage of using fixed features: it corresponds to deep learning (with actual multi-layer nets) in the so-called lazy regime \cite{jacot2018neural, chizat2018lazy}. In this case, the feature matrix is a random matrix. Therefore, our results go beyond linear models and are also relevant to deep learning in the lazy regime. In our numerical experiments, we shall thus not only work with the original data (see appendix \ref{numerical_simulations}, and in particular Fig.\ref{fig:all_datasets_finite_reg_direct}
), but also ---and mainly--- with random features maps and fixed features maps (as in Fig.\ref{fig:all_datasets_zero_reg} and Fig.\ref{fig:all_datasets_finite_reg})

For the labels, we shall focus in this work on the \emph{random label} model, where the $y_\mu$ are independent of the inputs $\bm{x}_\mu$, and generated independently according to a Rademacher distribution:
\begin{equation}
y_{\mu} \sim \frac{1}{2}\left( \delta_{-1} + \delta_{+1}\right).
\label{eq:y}
\end{equation}

In our theoretical approach, we shall use mainly two data models: 
\begin{itemize}[wide = 1pt]
\item The simplest one is the {\bf Gaussian covariate model} (GCM), where the inputs $\bm x_\mu\in\mathbb{R}^{p}$ are independently drawn from a Gaussian distribution:
\begin{equation}
\label{eq:def:gcm}
    {\bm x}_\mu \sim {\cal N}(\bm 0,\bm\Sigma).
\end{equation}
The Gaussian covariate model has been the subject of much attention recently ~\citep{dobriban2018high,thrampoulidis2018precise, hastie2019surprises,mei2019generalization,wu2020optimal,Bartlett30063,jacot2020kernel,celentano2020lasso,aubin2020generalization, loureiro2021learning, loureiro2022fluc}. In particular, the asymptotic statistics of the minimizer of eq.~\eqref{eq:def_min_problem_main} for different models for the labels can be computed using the replica method, and rigorously proven as well. In particular, the random label limit relevant to our discussion can be obtained as a limit of the expressions derived using the replica method of statistical physics and mathematically proven in \cite{loureiro2021learning}. We shall use here the following expressions (see also Appendix \ref{sec:app:gcm})

\begin{theorem}[Asymptotics of the GCM for random labels,  adapted from \cite{loureiro2021learning}, informal]
\label{theorem:GM}
 Consider the minimization problem in eq.~\eqref{eq:def_min_problem_main}, with the inputs $\bm{x}_\mu$ generated according to a Gaussian covariate model. Assume that the loss $\ell$ is strictly convex (or that $\ell$ is convex and $\lambda > 0$). Under mild regularity conditions on the $\bm \mu$, $\Sigma$, as well as the loss and regularizer, 
 then we have the asymptotic training performance of the empirical risk minimizer eq.~\eqref{eq:app:erm} for the random label Gaussian mixture model satisfying the scalings (\ref{assump:scalings}) in the proportional high-dimensional limit as $n \to \infty$:
\begin{equation}
\begin{split}
\widehat \cR_n^*(\bm X, y(\bm X)) &\overset{\dP}{\longrightarrow}   \mathcal{E}^{\text{gcm}}_{\ell}(\alpha, \lambda) \coloneqq  \frac{1}{2}\sum\limits_{y\in\{-1,+1\}}\mathbb{E}_{\xi\sim\mathcal{N}(0,1)}\left[\ell(\prox_{V^{\star}\ell(\cdot, y)}\left(\sqrt{q^{\star}}\xi\right), y)\right]
    \end{split}
\end{equation}
where $\prox_{\tau f(\cdot)}$ is the proximal operator associated with the loss:
\begin{equation}
    \prox_{\tau \ell(\cdot, y)}(x) \coloneqq  \underset{z\in\mathcal{C}_{1}}{\argmin}\left[\frac{1}{2\tau}(z-x)^2+\ell(z,y)\right]
\end{equation} and the parameters $(V^{\star}, q^{\star})$ are the (unique) fixed-point of the following self-consistent equations:
\begin{align}
&\begin{cases}
\hat{V} = \frac{\alpha}{2} \sum\limits_{y\in\{-1,+1\}}\mathbb{E}_{\xi\sim\mathcal{N}(0,1)}\left[\partial_{\omega}f_{\ell}(y, \sqrt{q}\xi, V)\right]\\
\hat{q} = \frac{\alpha}{2} \sum\limits_{y\in\{-1,+1\}}\mathbb{E}_{\xi\sim\mathcal{N}(0,1)}\left[f_{\ell}(y, \sqrt{q}\xi, V)^2\right]\\
\end{cases}, &&
\begin{cases}
	V &= \frac{1}{p}\tr\Sigma\left(\lambda\bm{I}_{p}+\hat{V}\Sigma\right)^{-1}\\
q &= \frac{1}{p}\hat{q}\tr\Sigma^2\left(\lambda\bm{I}_{p}+\hat{V}\Sigma\right)^{-2}\\
\end{cases}
\end{align}	
where $f_{g}(y,\omega,V) \coloneqq V^{-1}\left(\prox_{V\ell(\cdot, y)}(\omega)-\omega\right)$.
\end{theorem}

\item A more generic model of data, which has the advantage of being multi-modal to be fit complex situations, is the {\bf Gaussians Mixture Model} (GMM). In this case, the inputs $\bm{x}_{\mu}\in\mathbb{R}^{p}$ are independently generated as:
\begin{equation}
    {\bm x}_\mu \sim \sum_{c\in\mathcal{C}} \rho_{c}\, {\cal N}
    (\bm{\mu}_{c},\bm{\Sigma}_{c}) 
    \label{GMM-model}
\end{equation}
\noindent where $\mathcal{C} \coloneqq \{1,\cdots, K\}$ indexes the $K$ Gaussian clouds and $\rho_{c} \in[0,1]$ is the density of points in each cloud and satisfies $\sum_{c\in\mathcal{C}}\rho_{c} = 1$. The analysis of Gaussian mixture models in the high-dimensional regime has been the subject of many works. The exact asymptotic expression for the minimum training loss has been derived for a range of particular cases in, between others, \cite{kini2021phase,sifaou2019phase,mai2019large,mignacco2020role,taheri2020optimality,dobriban2020provable} and in full generality for arbitrary means and covariances in \cite{loureiro2021learning_gm}.  We shall thus use the random label limit of their expression in the binary classification case:

\begin{theorem}[Asymptotics of the GMM for random labels, adapted from \cite{loureiro2021learning_gm}, informal]\label{thm:app:asymp_error}    \label{propGMM}

 Consider the minimization problem in eq.~\eqref{eq:def_min_problem_main}, with the inputs $\bm{x}_\mu$ generated according to a Gaussian mixture as in \eqref{GMM-model}. Assume that the loss $\ell$ is strictly convex (or that $\ell$ is convex and $\lambda > 0$). Under mild regularity conditions on the $\bm \mu_c$, $\bm \Sigma_c$, as well as the loss and regularizer, we have the training performance of the empirical risk minimizer eq.~\eqref{eq:app:erm} for the random label Gaussian mixture model satisfying the scalings (\ref{assump:scalings}) are given by:
\begin{equation}\label{eq:app:asymptotic_risk}
\begin{split}
\widehat \cR_n^*(\bm X, y(\bm X)) &\overset{\dP}{\longrightarrow}  \mathcal{E}^{\text{gmm}}_{\ell}(\alpha, \lambda, K) \coloneqq  \frac{1}{2}\sum\limits_{c\in\mathcal{C}}\rho_{c}\sum\limits_{y\in\{-1,+1\}}\mathbb{E}_{\xi\sim\mathcal{N}(0,1)}\left[\ell(\prox_{V^{\star}_{c}\ell(\cdot, y)}\left(m_c^\star + \sqrt{q^{\star}_{c}}\xi\right), y)\right] 
    \end{split}
\end{equation}
\noindent where $\ell$ is the loss function used in the empirical risk minimization in eq.~\eqref{eq:app:erm}, $\prox_{\tau f(\cdot)}$ is the proximal operator associated with the loss:
\begin{equation}
    \prox_{\tau \ell(\cdot, y)}(x) \coloneqq  \underset{z\in\mathcal{C}_{1}}{\argmin}\left[\frac{1}{2\tau}(z-x)^2+\ell(z,y)\right]
\end{equation}
\noindent and $(m_c^\star, V_{c}^{\star},q^{\star}_{c})_{c\in\mathcal{C}}$ are the \textbf{unique} fixed points of the following self-consistent equations:
\begin{equation}
\begin{split}
\label{eq:app:sp:gmm}
&\begin{cases}
\hat{V}_{c} = \frac{\alpha}{2} \rho_{c}\sum\limits_{y\in\{-1,+1\}}\mathbb{E}_{\xi\sim\mathcal{N}(0,1)}\left[\partial_{\omega}f_{\ell}(y, m_{c}+\sqrt{q_{c}}\xi, V_{c})\right]\\
\hat{q}_{c} = \frac{\alpha}{2} p_{c}\sum\limits_{y\in\{-1,+1\}}\mathbb{E}_{\xi\sim\mathcal{N}(0,1)}\left[f_{\ell}(y, m_{c}+\sqrt{q_{c}}\xi, V_{c})^2\right]\\
\hat{m}_{c} = \frac{\alpha}{2} p_{c}\sum\limits_{y\in\{-1,+1\}}\mathbb{E}_{\xi\sim\mathcal{N}(0,1)}\left[f_{\ell}(y, m_{c}+\sqrt{q_{c}}\xi, V_{c})\right]
\end{cases}\\
&
\begin{cases}
	V_{c} &= \frac{1}{p}\tr\Sigma_{c}\left(\lambda\bm{I}_{p}+\sum\limits_{c'\in\mathcal{C}}\hat{V}_{c'}\Sigma_{c'}\right)^{-1}\\
q_{c} &= \frac{1}{p}\sum\limits_{c'\in\mathcal{C}}\left[\tr\left(\hat{q}_{c'}\Sigma_{c'}+\hat{m}_{c}\hat{m}_{c'}\bm{\mu}_{c'}\bm{\mu}_{c}^{\top}\right)\Sigma_{c}\left(\lambda\bm{I}_{p}+\sum\limits_{c''\in\mathcal{C}}\hat{V}_{c''}\Sigma_{c''}\right)^{-2}\right]\\
m_{c} &= \frac{1}{p}\sum\limits_{c'\in\mathcal{C}}\hat{m}_{c}\hat{m}_{c'}\left[\tr\bm{\mu}_{c'}\bm{\mu}_{c}^{\top}\left(\lambda\bm{I}_{p}+\sum\limits_{c''\in\mathcal{C}}\hat{V}_{c''}\Sigma_{c''}\right)^{-1}\right]
\end{cases}
\end{split}
\end{equation}	
\noindent where $f_{\ell}(y,\omega,V) \coloneqq V^{-1}\left(\prox_{V\ell(\cdot, y)}(\omega)-\omega\right)$. 
\end{theorem}

\end{itemize}

\section{The main theoretical results: from mixtures to a single Gaussian} 
In this section, we present the main theoretical results of the present work and discuss their consequences: We show that with random labels, GMM models can be  reduced to a single GCM model. This provides an explanation of the universality observed in Figs.~\ref{fig:all_datasets_zero_reg} and \ref{fig:all_datasets_finite_reg}.

We would like the reader to note that we state our results using  theorems because indeed we were able to establish them mathematically rigorously. However, the proofs are deferred to the appendices and the reasoning and derivations presented in this section follow the level of rigour common in theoretical physics. We made this choice to ensure readability to both physics and mathematics-oriented audiences. 

The starting point is the Gaussian Mixture Model (GMM). This is a very generic model of data and standard approximation results (e.g. the Stone-Weierstrass theorem) show in particular that one can approximate data density to arbitrary precision by Gaussian mixtures. While, in the worst case this would require a diverging number of Gaussian in the mixture, it can be shown that (as far as the generalized linear model is concerned) a mixture of a small number of Gaussians is actually able to approximate very complex data set in high-dimension \cite{seddik_2020_random,louart2018random,seddik_2021_unexpected}. More precisely, in the proportional high-dimensional regime, data generated by Generative Adversarial Networks (GAN), one of the state of the art techniques to generate realistic looking data, 
behave as Gaussian mixtures for such classifiers \cite{futureuniversality}. We shall thus use this model as our benchmark of ``complex'' data distribution.

If a mixture model is a good approximation of reality in high-dimension, the question remains: {\it why is it that we can fit real dataset with a single Gaussian}. Our main technical question will thus be: if we use random labels, what is the difference between a GMM and a single Gaussian model? 

\subsection{Mean invariance with random labels}
We thus now move to the random label case and show how we can surprisingly use a simple Gaussian distribution instead of the GMM. We are going to use theorem \ref{theorem:GM} and \ref{thm:app:asymp_error}. Note that the asymptotic value of the energy, or loss, only depends on the probability vector $\bm \rho\in[0,1]^{K}$ (with entries $\rho_{c}$ corresponding to the respective sizes of the $K$ clusters), the matrix of averages $\bm M\in\mathbb{R}^{K\times p}$ (with rows $\bm{\mu_{c}}\in\mathbb{R}^{p}$), and the concatenation of covariances $\bm \Sigma^{\otimes}\in\mathbb{R}^{K\times p\times p}$ (with rows $\bm \Sigma_{c}\in\mathbb{R}^{p\times p}$) and therefore we denote:
\[ \mathcal{E}_{\ell} = \mathcal{E}_{\ell}^{\text{gmm}}(\bm \rho, \bm M, \bm \Sigma^\otimes). \]
Similarly, for the Gaussian covariate model we define the limiting value
\[ \mathcal{E}_{\ell} = \mathcal{E}_{\ell}^{\text{gcm}}(\bm m, \bm \Sigma). \]
where in both cases we omitted the explicit dependence on the parameters $(\alpha, \lambda)$. We are now in a position to state a lemma crucial to our first main universality result:
\begin{lemma}[Single mean lemma for random labels]\label{lemma:mean_universality}
    In the random label setting \eqref{eq:y}, assume that the loss $\ell$ is symmetric, in the sense that $\ell(x, y) = \ell(-x, -y)$ for $x, y \in \dR$. Then, the limiting value $\mathcal{E}_{\ell}$ of the risk is independent from the means, i.e.  for all choices of $\bm\rho$, $\bm M$ and $\bm\Sigma^\otimes$ we have
    \[ \mathcal{E}_{\ell}^{\text{gmm}}(\bm \rho, \bm M, \bm \Sigma^\otimes) = \mathcal{E}_{\ell}^{\text{gmm}}(\bm \rho, \bm 0, \bm \Sigma^\otimes)\,. \]
\end{lemma}

The symmetry condition on the loss is not really restrictive and is satisfied by virtually all losses used in binary classification (in particular margin-based losses of the form $\ell(x, y) = \phi(xy)$). Since a mixture of Gaussians with equal means and covariances is equivalent to a single Gaussian, we can now write the following theorem:
\begin{theorem}[Gaussian universality for random labels]
\label{thm:Gauss_universality}
Consider the same assumptions as in Lemma \ref{lemma:mean_universality}, and assume further that the data is homogeneous, i.e.
\[ \bm\Sigma_{c} = \bm \Sigma \quad \text{for all} \quad c\in\mathcal{C}. \]
Then the asymptotic risk is equivalent to that of a single centered Gaussian:
\[ \mathcal{E}_{\ell}^{\text{gmm}}(\bm \rho, \bm M, \bm \Sigma^\otimes) = \mathcal{E}_{\ell}^{\text{gcm}}(\bm 0, \bm \Sigma). \]
\end{theorem}
This is our first main universality theorem: a mixture of homogeneous Gaussians
\footnote{Also called homoskedastic Gaussians, as opposed to heteroskedastic Gaussians} 
can be replaced, when using random labels by a single Gaussian. 

This surprising fact, alone, explains the empirical observation presented in Fig.~\ref{fig:all_datasets_zero_reg} and Fig.~\ref{fig:all_datasets_finite_reg}, at least if we accept that the different modes are homogeneous (see discussion in Sec.\ref{sec:numerics}). 

\paragraph*{Proof sketch ---}
Both lemma \ref{lemma:mean_universality} and Theorem \ref{thm:Gauss_universality} stem from the detailed analysis of the replica free energy for the GMM \cite{loureiro2021learning_gm}. Indeed, to prove our claims, it suffices to show that the fixed points of the replica equations are the same. This is done in detail in appendix \ref{proof2}, using the replica equation that we provide in appendix \ref{formulas}. In a nutshell, we show that the expression of the GMM reduces to those of the GCM. \hfill \qedsymbol

%

\subsection{Generic loss with vanishing regularisation}

Additionally, we note that in Fig.~\ref{fig:all_datasets_zero_reg} at vanishing regularization, we did not even require a matching covariance, and instead used a trivial one. This is because of the following consequence of Lemma \ref{lemma:mean_universality}:\looseness=-1
\begin{theorem}[Gaussian universality for vanishing regularization]
\label{theorem:cov}
Consider the same assumptions as in theorem  \ref{thm:Gauss_universality}, then if the minimizer of $\ell$ is unique and the data covariance full-rank, then the asymptotic minimal loss for Gaussian data does not depend on the covariance when the regularization is absent, $\lambda =0$.
\end{theorem}
\begin{proof}
 Consider the empirical risk minimization problem in eq.~\eqref{eq:def_min_problem_main} with data from the Gaussian covariate model eq.~\eqref{eq:def:gcm} with random labels.  Without loss of generality, we can write $\bm{x}_{\mu} = \bm \Sigma^{1/2}\bm{z}_{\mu}$, with ${\bm z}_{\mu}\sim\mathcal{N}(\bm{0}_{p},\bm{I}_{p})$. Then, making a change of variables $\bm \theta'=\bf \Sigma^{1/2} \bm \theta$, we can write: 
 \begin{align}
\widehat\cR_n^*(\bm X, \bm y) = \inf_{\bm \theta} \frac1n\sum_{\mu=1}^n \ell(\bm \theta^\top \bm x_\mu, y_\mu) +  \frac{\lambda}{2} \vert\vert \bm{\bm \theta} \vert\vert_2^2 = \inf_{\bm \theta' \in \cS'_p} \frac1n\sum_{\mu=1}^n \ell(\bm \theta'^\top \bm z_\mu, y_\mu) +  \frac{\lambda}{2} \vert\vert \Sigma^{-1/2}\bm{\bm \theta}' \vert\vert_2^2\, \notag
\end{align}
\noindent where $\cS'_p\subset\mathbb{R}^{p}$ is another compact set, and we have used the fact that $y_{\mu}$ are independent of $\bm{x}_{\mu}$. Since the minimizer of $\ell$ is unique, the result follows from taking $\lambda\to 0^{+}$.
\end{proof}
Note that in particular theorem \ref{theorem:cov} implies that for random labels, the GCM model with a covariance $\Sigma$ is equivalent to a Gaussian i.i.d. model with a different regularization given by the norm  $\vert\vert \cdot\vert\vert_{\Sigma^{-1}}$ induced by the inverse covariance matrix $\Sigma^{-1}$. Therefore, in the case in which $\ell$ has several minima, the $\lambda\to 0^{+}$ limit will give the performance of the solution with minimum $\vert\vert \cdot\vert\vert_{\Sigma^{-1}}$ norm. 

Finally, we also note that this analysis also allows answering the important question: {\it what is being learned with random labels?}, discussed in  particular in machine learning literature in \cite{maennel2020neural}. For generalized linear models: the model is simply fitting the 2nd-order statistics (the total covariance $\bf \Sigma$). 

\subsection{Ridge regression with vanishing regularization} Even though it seems to be well obeyed in practice, one may wonder if we can in some cases get rid of the homogeneity condition. As we shall see, the answer is no: in general, a mixture of {\it inhomogeneous } Gaussian cannot be strictly replaced by a single one. It turns out, however, that there is one exception, and that the hypothesis can be lifted in one case, ridge regression with vanishing regularization with the squared loss $\ell(x, y) = \frac12(x-y)^2$:
\begin{theorem}[Strong universality for ridge loss]
\label{thm:ols}
    In the ridge regression case with vanishing regularization, i.e. when $\lambda \to 0^+$, we have
    \[\lim\limits_{\lambda\to0^{+}} \mathcal{E}_{\ell}^{\text{gmm}}(\bm \rho, \bm M, \bm \Sigma^\otimes) = \frac12 \left(1  - \frac{1}{\alpha} \right)_+, \]
    for any choice of $\bm \rho, \bm M$, or $\bm \Sigma^\otimes$.
\end{theorem}
In particular, it means that in the unregularized limit, any Gaussian mixture behaves in terms of its loss as a single cluster Gaussian model with identity covariance, whose asymptotic training loss is given by $\lim_{\lambda\to 0^{+}}\mathcal{E}_{\ell}^{\text{gcm}}(\alpha, \lambda)=\sfrac{1}{2}(1-1/\alpha)_{+}$.

\paragraph*{Proof sketch ---}
The proof of the strong universality, that follows from a rigorous analysis of the replica predictions, amounts to showing that the replica free energy for GMM reduces to the one of a single Gaussian. Interestingly, although the fixed points of the replica equations differ between the GMM and Gaussian case, they do give rise to the same free energy.
Details can, again, be found in Appendix \ref{sec:app:uni}. \hfill \qedsymbol

\section{Numerical experiments}
\label{sec:numerics}

In this section, we  describe more in detail the numerical experiments of   Fig.~\ref{fig:all_datasets_zero_reg} \& Fig.~\ref{fig:all_datasets_finite_reg}. The coloured dots represent the outcome of the simulations on several full-rank datasets. In particular, the blue and the green dots refer to both MNIST and grayscale CIFAR-10 preprocessed with random Gaussian feature maps \citep{rahimi2007random}. In this case, the input data points are constructed as $\bm{x}_{\mu} = \sigma\left( \bm{z}_{\mu} \bm{F}  \right)$, with $\bm{z}_{\mu} \in \mathbb{R}^d$ being a sample from one of the two datasets, $\bm{F} \in \mathbb{R}^{d\times p}$ representing the matrix of random features, whose row elements are sampled according to a normal distribution, and $\sigma$ being some point-wise non-linearity, namely $\mathrm{erf}$ for MNIST and $\mathrm{relu}$ for grayscale CIFAR-10. The red dots correspond instead to fashion-MNIST pre-processed with wavelet scattering transform, an ensemble of engineered features producing rotational and translational invariant representations of the input data points \cite{bruna2013invariant}. The orange dots correspond to simulations on the synthetic dataset built as a mixture of two Gaussians, with data covariance of the two clusters both equal to the identity matrix ($\bm{\Sigma}_1 \!=\! \bm{\Sigma}_2 \!=\! \bm{I}$,  $\bm{\mu}_{1/2} \!=\! \left( \pm 1,0,...,0 \right)$ and $\rho_{1/2} \!=\! 1/2$. Further technical details are given in the appendix.

\paragraph{Experiments with finite regularization ---} Fig.~\ref{fig:all_datasets_finite_reg} illustrates the Gaussian universality taking place at finite regularization. The coloured dots correspond to the outcome of the simulations for several values of the regularization strength. As we can see from this set of plots, the theoretical learning curve of a single Gaussian with matching covariance perfectly fits the behaviour of multi-modal and more realistic input data distributions. In fact, even though the experiment is performed for a realistic dataset and finite $n$ and $p$, the asymptotic Gaussian theory gives a perfect fit of the data.

\paragraph{Experiments with vanishing regularization ---} Fig.~\ref{fig:all_datasets_zero_reg} provides an illustration of the universality effect occurring at vanishing regularization. Here we use $\lambda \!\to\!0$, and following theorem \ref{theorem:cov}, we observe a collapse on a single curve given by the asymptotic theory for a single Gaussian with unit covariance. It is  quite remarkable that our asymptotic theory, which is valid only in the infinite high-dimensional limit, is validated by such experiments with finite dimension, and finite sample size.\looseness=-1

\paragraph{Homogeneity assumption ---} A remarkable point is that the homogeneity assumption (often called homoskedasticity in statistics) we use in theorem \ref{thm:Gauss_universality}, which can be relaxed only for ridge regression, does not seem to be that important in practice, as we observed on our experiments on real data. One may thus wonder if the strong universality of Theorem \ref{thm:ols} could be proved in full generality, and not only for the ridge loss. It turns out that the answer is no. Using Proposition~\ref{propGMM}, we can actually construct an artificial mixture of Gaussians, using {\it very different} covariances for each individual Gaussians, and observe small deviations from the strict universality. A mixture of {\it non-homogeneous} Gaussians is not strictly equivalent to a single one with random labels (except, as stated in Theorem \ref{thm:ols}, for the least squares that obey a strong universality). This is illustrated in Fig.~\ref{fig:no_match_reg_only_three_blocks} where we show the disagreement in the behaviour of the training loss between a single Gaussian and a mixture of two non-homogeneous Gaussians.
This is a simple counter-example to the existence of a universal strong form of Gaussian universality, even for ridge regression (see discussions in e.g. \cite{goldt2020gaussian,bordelon2020spectrum,tomasini2022failure, Ingrosso2022, Ba2022}).\looseness=-1

It may thus come as surprising that real datasets, which certainly will not obey such a strict homogeneity of the different modes, display such a spectacular agreement with the theory. We believe that this is due to two effects: first, the deviations we observed, even in our designed counter-example, are small, so they might not even be seen in practice. Secondly, and especially after observing the data through random or scattering features, it turns out that when we measure the empirical correlation matrix of the different modes, they look quite similar. 
In fact, it has even been suggested that neural networks are {\it precisely} learning representations that find such homogeneous Gaussian mixtures \cite{papyan2020prevalence}.


\begin{figure}[t!]
\begin{center}
\centerline{\includegraphics[width=500pt]{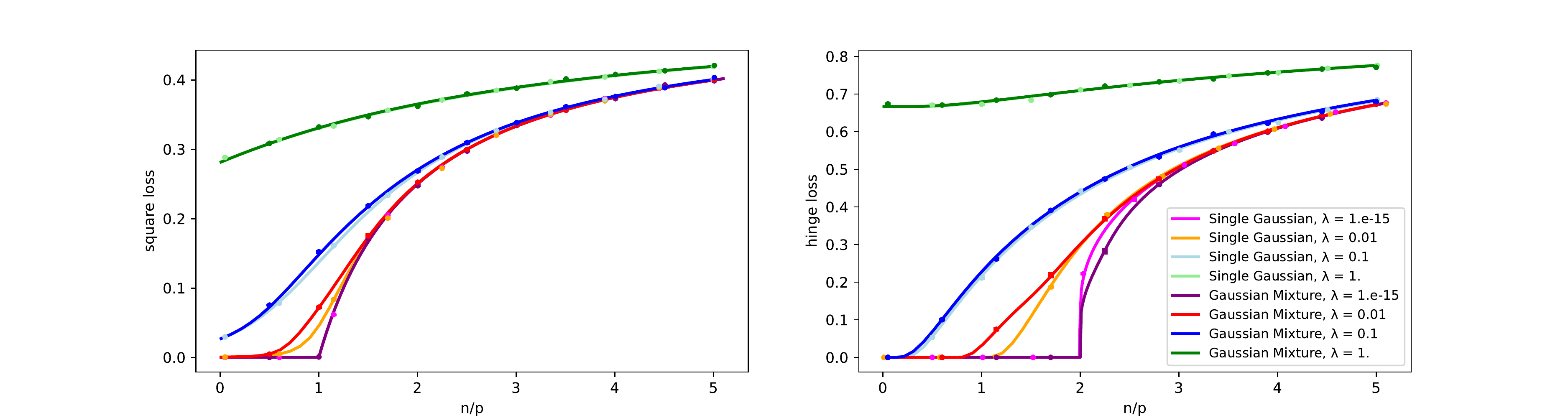}}
\caption{Ridge/square loss (left) \& 
hinge loss for a single Gaussian vs a mixture of inhomogeneous Gaussians at finite $\lambda$. Lines are the asymptotic exact results while dots are simulation ($p\!=\!900$, dark lines for mixture, lighter ones for single Gaussian). When the homogeneity assumption is not obeyed, then a mixture of two Gaussians does not yield
equal results to those of a single Gaussian with matching covariance. (Here, a mixture with zero mean and a block covariance with, resp. diagonal elements equal to $0.01$, $0.98$ and $0.01$ for the first one, and $0.495$ and $0.01$, $0.495$ for the second). Note however that the universality is restored in the Ridge case when $\lambda \to 0$, as stated in Theorem \ref{thm:ols}. It is also very well obeyed with large enough $\lambda$ and deviations appear small in general.}
\label{fig:no_match_reg_only_three_blocks}
\end{center}
\vskip -0.2in
\end{figure}

\paragraph{A remark on Rademacher complexity ---} A final comment is that the discussed universality indicates that, in high dimension, the Rademacher complexity can be effectively replaced by the one for Gaussian i.i.d. data. \emph{Rademacher complexity} is a key quantity appearing in generalization bounds for binary classification problems \cite{shalev2014understanding,vapnik1999nature} that measures the ability of estimators in a hypothesis class $\mathcal{H}$ to fit i.i.d. random labels $y^{\mu}\sim \text{Rad}(\sfrac{1}{2})$:
\begin{equation}
    \text{Rad}_n\left( \mathcal{H}\right) = \mathbb{E} \left[ \underset{h\in \mathcal{H}}{\mbox{sup}} \ \frac{1}{n} \sum_{\mu = 1}^n y_{\mu} h\left( \bm{x}_{\mu} \right) \right].
\end{equation}
It is explicitly dependent on the specific distribution of the input data points $\bm{x}_{\mu}$. As discussed in \citep{abbaras2020rademacher} there exists a direct mapping between the Rademacher complexity and the minimum 0/1 training loss - or ground state energy in the statistical physics parlance. Indeed, for a binary hypothesis class $\mathcal{H} = \{h:\mathbb{R}^{p}\to\{-1,+1\}\}$ the two are asymptotically related by the following equation:
\begin{equation}
\underset{n \rightarrow \infty}{\lim}\underset{h\in\mathcal{H}}{\inf}\frac{1}{n}\sum\limits_{\mu=1}^{n}\mathbb{P}\left(h(\bm{x_{\mu}})\neq y_{\mu}\right) = \frac{\alpha}{2}\left[ 1 - \text{Rad}_n\left( \mathcal{H}\right)  \right].
\end{equation}
Moreover, \cite{abbaras2020rademacher} discussed how to explicitly compute the Rademacher complexity for Gaussian data using the replica method from statistical physics. This is actually a classical problem, studied by the pioneers of the application of the replica method and spin glass theory to theoretical machine learning~\citep{gardner1988optimal,gardner1989three,krauth1989storage,derrida1991finite}.
Given the universality advocated in this present work, these Gaussian results thus seem to be of more relevance than previously thought, and in fact, allow us to compute a closed-form asymptotic expression for the Rademacher complexity for realistic data. This is a very interesting outcome of the Gaussian universality with random labels.

However, while we prove universality for convex losses, we so far only {\it conjecture it} for non-convex objectives, such as the ones appearing in the definition of the Rademacher complexity. 
The proof that a Gaussian mixture approximates well real datasets is still valid for non-convex losses. The identification of these mixtures to a single Gaussian is, however, using the replica formulas of \cite{loureiro2021learning, loureiro2021learning_gm} which have been proven only for the case of convex losses. Our conjecture thus depends on proving a similar result for non-convex (as well as replica symmetry breaking) losses. This (and similar questions on multi-layer networks) is left for future work.

\section{Conclusion}  
For the classical problem of fitting random labels with perceptrons aka generalized linear models in high dimensions, we showed that, far from being only a toy example, the Gaussian i.i.d. assumption is an excellent model of reality. The conclusion extends to deep-learning models in the lazy regimes as those are essentially random feature models. 
There are a number of potentially interesting extensions of this work, including non-convex losses and multi-layer neural networks, and beyond the random label cases, that should be investigated in the future. 

These results, we believe, are of special interest given the number of theoretical studies with the Gaussian design and its variants, that are amenable to exact characterization, and that turn out to be less idealistic, and more realistic, than perhaps previously assumed. We believe, in particular, that these considerably strengthen the ensemble of results obtained within the statistical physics community, as well as in the statistical analysis of high-dimensional data. We anticipate that such redemption of the Gaussian assumption will lead to more work in this direction both those using Gaussian assumption and those aiming to extend out universality results. 

\section*{Acknowledgements}
We acknowledge funding from the ERC under the European Union’s Horizon 2020 Research and Innovation Program Grant Agreement 714608-SMiLe, as well as by the Swiss National Science Foundation grant SNFS OperaGOST, $200021\_200390$ and the \textit{Choose France - CNRS AI Rising Talents} program.


\bibliographystyle{unsrt}
\bibliography{bib}
\clearpage

\newpage
\appendix
\newpage

\section{Exact asymptotic performances of GCM and GMM}
\label{formulas}
In this appendix we summarize the exact asymptotic formulas for the performance of the generalized linear classifiers on random labels for the two structured data models studied in the the main body: the Gaussian covariate model (GCM) and the Gaussian mixture model (GMM).

\subsection{Preliminaries: the setting}
\label{sec:app:preliminaries}
Before moving to the key formulas, let us recap the setting. We are interested in the performance of generalised linear classifiers:
\begin{align}
    \hat{y}(\bm{x}) = \text{sign}(\hat{\bm{\theta}}^{\top}\bm{x})
\end{align}
where $\hat{\bm{\theta}}\in\mathbb{R}^{p}$ is trained by minimising the following empirical risk on $n$ independent training samples $(\bm{x}_{\mu}, y_{\mu})_{\mu\in[n]}\in\mathbb{R}^{p}\times\{-1,+1\}$:
\begin{align}
\label{eq:app:erm}
 \widehat\cR_n^*(\bm X, \bm y) = \inf_{\bm \theta \in \cS_p} \frac1n\sum_{\mu=1}^n \ell(\bm \theta^\top \bm x_\mu, y_\mu) + 
 \frac{\lambda}{2} \vert\vert \mathbf{\bm \theta} \vert\vert_2^2 ,
 \end{align}
 for a compact subset $\cS_p\subset\mathbb{R}^{p}$ and convex loss function $\ell$. In particular, we are interested in the case where the labels $y_{\mu}\in\{-1,+1\}$ are randomized (i.e. not correlated with the inputs $\bm{x}_{\mu}$),
 \begin{align}
     y_{\mu} \sim \frac{1}{2}\left( \delta_{-1} + \delta_{+1}\right), \qquad \text{i.i.d.}
 \end{align}
 and the inputs are generated independently from one of the following two structured models:
 \begin{description}
 \item[Gaussian covariate model (GCM): ] $\bm{x}_{\mu}\sim\mathcal{N}(\bm{0}_{p}, \bm{\Sigma})$,
 \item[Gaussian mixture model (GMM): ] ${\bm x}_\mu \sim \sum_{c\in\mathcal{C}} \rho_{c}\, {\cal N}
    (\bm \mu_{c},\bm\Sigma_{c})$,
 \end{description}
where $\mathcal{C} = \{1,\cdots, K\}$ is the label set for the Gaussian clouds and $\rho_{c}\in[0,1]$ are the density of points in each class, satisfying $\sum_{c\in\mathcal{C}}\rho_{c} = 1$. Note that in this random label setting the GCM model is a special case of the GMM, where $K\coloneqq |\mathcal{C}| = 1$ and $\bm{\mu}_{1}=\bm{0}_{p}$.

In the following, we will be interested in describing the exact asymptotic limit of the following performance metrics in the proportional high-dimensional limit where $n,p\to\infty$ with the ratio $\alpha \coloneqq \sfrac{n}{p}$ and the number of clusters $K$ are fixed:
\begin{description}
\item[Training loss: ] $\hat{\mathcal{E}}_{\ell}(\bm{X},\bm{y}) \coloneqq \frac{1}{n}\sum\limits_{\mu=1}^{n}\ell\left(\hat{\bm{\theta}}^{\top}\bm{x}_{\mu}, y_{\mu}\right)$
\item[0/1 training error: ] $\hat{\mathcal{E}}_{0/1}(\bm{X},\bm{y}) \coloneqq \frac{1}{n}\sum\limits_{\mu=1}^{n}\mathbb{P}\left(\text{sign}(\hat{\bm{\theta}}^{\top}\bm{x}_{\mu})\neq y_{\mu}\right)$
\end{description}
\noindent where we have defined the design matrix $\bm{X}\in\mathbb{R}^{p\times n}$ and the label vector $\bm{y}\in\{-1,+1\}^{n}$. Note that for convenience we will focus the discussion in this appendix to these two measures. But all results could have been stated for  $\hat{\mathcal{R}}^{\star}_{n}$ instead. In particular, the training loss $\hat{\mathcal{E}}_{\ell}$ differs from the empirical risk $\hat{\mathcal{R}}^{\star}_{n}$ by the regularisation term.

\paragraph{Note on scalings --} Although the model above is well defined for any scaling, in the following we focus in the case in which the means are covariance satisfy:
\begin{align}
    ||\bm{\mu}_{c}||_{2}^{2} = O(1), && \tr\bm{\Sigma}_{c} = O(p).
    \label{assump:scalings}
\end{align}
This scaling of the mean and variance is indeed the natural one (see e.g. \cite{barkai1994statistical,lesieur2016phase,lelarge2019asymptotic,mignacco2020role,wang2021benign}) as well as the most interesting in high-dimensions. If the means have larger norm, then the problem becomes trivial (i.e. the Gaussians are trivially completely separable) while if the means are smaller it is  impossible to separate them (i.e. they become trivially indistinguishable from a single Gaussian cloud).

\paragraph{Ridge and ordinary least-squares classification --} Note that for the special case of the ridge classification in which $\ell(x,y) = \sfrac{1}{2}(y-x)^2$, the empirical risk minimization problem defined in eq.~\eqref{eq:app:erm} admits a closed form solution:
\begin{align}
\label{eq:app:ridge}
    \hat{\bm{\theta}} = \left(\lambda\bm{I}_{p}+\bm{X}\bm{X}^{\top}\right)^{-1}\bm{X}\bm{y}
\end{align}
\noindent and therefore the computation of the asymptotic training error or loss boils down to a Random Matrix Theory problem, with a solution equivalent to the one we will discuss shortly below. However, some qualitative features can be drawn just from this expression. First, note that for $\lambda>0$, the ridge estimator above will always have a non-zero training loss because of the bias introduced by the regularization term $\sfrac{1}{2}\lambda||\bm{\theta}||_{2}^{2}$. This can only be achieved in the limit of vanishing regularization $\lambda\to 0^{+}$, in which case the ridge estimator simplifies to:
\begin{align}
\label{eq:app:ln}
    \hat{\bm{\theta}}_{\text{ols}} \coloneqq (\bm{X}^{\top})^{\dagger}\bm{y}
\end{align}
where $\bm{X}^{\dagger}\in\mathbb{R}^{n\times p}$ is the Moore-Penrose inverse of $\bm{X}$. In the simplest case in which $\bm{X}$ is a full-rank matrix (which ultimately depends on the covariances), it can be explicitly written as:
\begin{align}
    \bm{X}^{\dagger} \coloneqq 
        \begin{cases}
        (\bm{X}^{\top}\bm{X})^{-1}\bm{X}^{\top} \text{ if } \alpha < 1\\
        \bm{X}^{\top}(\bm{X}\bm{X}^{\top})^{-1} \text{ if } \alpha > 1
        \end{cases}
\end{align}
An important property of the estimator in eq.~\eqref{eq:app:ln} is that it corresponds to the least $\ell_2$-norm interpolator when the system is underdetermined. Indeed, in the strict case when $\lambda=0$ (i.e. least-squares regression) the ERM problem in eq.~\eqref{eq:app:erm} is equivalent to inverting a linear system:
\begin{align}
    \bm{y} = \bm{X}^{\top}\bm{\theta}
\end{align}
i.e. to solve a system of $n$ equations for $p$ unknowns. Again, assuming the data is full-rank\footnote{The general case is given by changing $p$ for the rank of the design matrix.}, for $\alpha = \sfrac{n}{p} < 1$ the system is \emph{underdetermined}, meaning that there are infinitely many solutions that perfectly interpolate the data. Among all of them, $\hat{\bm{\theta}}_{\text{ols}}$ is the one that has lowest $\ell_{2}$-norm. Instead, when $\alpha>1$, the system is overdetermined, and no interpolating (zero-loss) solution exists.

\subsection{Gaussian mixture model with general labels}
\label{sec:app:gmm}
Exact asymptotics of generalized linear classification with Gaussian Mixtures in the proportional regime have been derived under different settings in the literature \cite{kini2021phase,sifaou2019phase,mai2019large,mignacco2020role,taheri2020optimality,dobriban2020provable}. Of particular interest to our work are the formulas proved in \cite{loureiro2021learning_gm} under the most general setting of a multi-class learning problem with convex losses \& penalties and generic means and covariances. In their work, the asymptotic performance of the minimiser in eq.~\eqref{eq:app:erm} was proven in the case where the labels are correlated to the mean. The formula we state in the text as theorem \ref{thm:app:asymp_error} is a straightforward adaptation of their result in the particular case of binary classification with $K$ clusters and randomized labels.

\paragraph{Zero mean limit:} Of particular interest for what follows is the zero-mean limit $\bm{\mu_{c}}=\bm{0}_{p}$ of the above equations, which is simply given by:
\begin{equation}\label{eq:app:sp:gmm_zeromean}
\begin{split}
&\begin{cases}
\hat m_c = 0\\
\hat{V}_{c} = \frac{\alpha}{2} \rho_{c}\sum\limits_{y\in\{-1,+1\}}\mathbb{E}_{\xi\sim\mathcal{N}(0,1)}\left[\partial_{\omega}f_{\ell}(y, \sqrt{q_{c}}\xi, V_{c})\right]\\
\hat{q}_{c} = \frac{\alpha}{2} p_{c}\sum\limits_{y\in\{-1,+1\}}\mathbb{E}_{\xi\sim\mathcal{N}(0,1)}\left[f_{\ell}(y, \sqrt{q_{c}}\xi, V_{c})^2\right]\\
\end{cases}\\
&
\begin{cases}
    m_c = 0\\
	V_{c} = \frac{1}{p}\tr\Sigma_{c}\left(\lambda\bm{I}_{p}+\sum\limits_{c'\in\mathcal{C}}\hat{V}_{c'}\Sigma_{c'}\right)^{-1}\\
q_{c} = \frac{1}{p}\sum\limits_{c'\in\mathcal{C}}\left[\hat{q}_{c'}\tr\Sigma_{c'}\Sigma_{c}\left(\lambda\bm{I}_{p}+\sum\limits_{c''\in\mathcal{C}}\hat{V}_{c''}\Sigma_{c''}\right)^{-2}\right]\\
\end{cases}
\end{split}
\end{equation}

\paragraph{A particular case: ridge classification --} The self-consistent equations above crucially depend on the loss function $\ell$. A case particular case of interest in this work - and for which the expressions considerable simply - is the case of ridge regression where $\ell(x,y)=\sfrac{1}{2}(x-y)^2$. In this case, the proximal can be explicitly writen as:
\begin{align}
    \prox_{\tau \ell(\cdot, y)}(x) = \frac{x+\tau y}{1+\tau} \quad\Leftrightarrow\quad  f_{\ell}(y,\omega,V) = \frac{y-\omega}{1+V}	
\end{align}
\noindent and therefore the asymptotic training loss admits a closed-form expression:
\begin{align}
\label{eq:app:loss:ridge:gmm}
    \mathcal{E}^{\text{gmm}}_{\ell} = \sum\limits_{c\in\mathcal{C}}\rho_{c}\frac{1+q^{\star}_{c}}{2(1+V^{\star}_{c})^2}
\end{align}
for $(V_{c}^{\star},q^{\star}_{c})_{c\in\mathcal{C}}$ solutions of the following simplified self-consistent equations:
\begin{align}
&\begin{cases}
\hat{V}_{c} = \frac{\alpha \rho_{c}}{1+V_{c}}\\
\hat{q}_{c} = \alpha p_{c}\frac{1+q_{c}}{(1+V_{c})^2}\\
\end{cases}, &&
\begin{cases}
	V_{c} &= \frac{1}{p}\tr\Sigma_{c}\left(\lambda\bm{I}_{p}+\sum\limits_{c'\in\mathcal{C}}\hat{V}_{c'}\Sigma_{c'}\right)^{-1}\\
q_{c} &= \frac{1}{p}\sum\limits_{c'\in\mathcal{C}}\left[\hat{q}_{c'}\tr\Sigma_{c'}\Sigma_{c}\left(\lambda\bm{I}_{p}+\sum\limits_{c''\in\mathcal{C}}\hat{V}_{c''}\Sigma_{c''}\right)^{-2}\right]
\end{cases}
\end{align}
Note that in particular, at the fixed point, we can also express the training loss eq.~\eqref{eq:app:loss:ridge:gmm} as:
\begin{align}
\label{eq:app:loss:ridge:gmm-flo}
    \mathcal{E}^{\text{gmm}}_{\ell} = \sum\limits_{c\in\mathcal{C}} \frac{\hat q^*_c}{2 \alpha}.
\end{align}

\subsection{Gaussian covariate model}
\label{sec:app:gcm}
The asymptotic training loss for the Gaussian covariate model for a fairly general teacher-student setting was first proven in \cite{loureiro2021learning}. Although the random label limit can be obtained from this work, as discussed in Sec.~\ref{sec:app:preliminaries} the random label Gaussian covariate model can also be seen as a particular case of the general Gaussian mixture model with $K=1$ and $\bm{\mu}_{1}=\bm{0}_{p}$. Therefore, its asymptotic performance is included in the discussion above. This lead to theorem \ref{theorem:GM} in the main text.

It is worth noting that, for the square loss the expressions simplify considerably. The training loss is given by:
\begin{align}
    \mathcal{E}^{\text{gmm}}_{\ell} = \frac{1+q^{\star}}{2(1+V^{\star})^2}
\end{align}
where $(V^{\star},q^{\star})$ are solutions of the following simplified self-consistent equations:
\begin{align}
&\begin{cases}
\hat{V} = \frac{\alpha}{1+V}\\
\hat{q} = \alpha \frac{1+q}{(1+V)^2}\\
\end{cases}, &&
\begin{cases}
	V &= \frac{1}{p}\tr\Sigma\left(\lambda\bm{I}_{p}+\hat{V}\Sigma\right)^{-1}\\
q &= \frac{1}{p}\hat{q}\tr\Sigma^2\left(\lambda\bm{I}_{p}+\hat{V}\Sigma\right)^{-2}
\end{cases}
\end{align}
Since the covariance $\Sigma$ is positive-definite (and therefore invertible), in the overdetermined regime (for which the training loss is non-zero), the limit $\lambda\to 0^{+}$ can be easily taken, and the equations reduce to:
\begin{align}
&\begin{cases}
\hat{V} = \frac{\alpha}{1+V}\\
\hat{q} = \alpha \frac{1+q}{(1+V)^2}\\
\end{cases}, &&
\begin{cases}
	V &= \frac{1}{\hat{V}}\\
q &= \frac{\hat{q}}{\hat{V}}
\end{cases}
\end{align}
\noindent which is completely independent of the covariance matrix $\Sigma$\footnote{As discussed in Theorem \ref{theorem:cov}, the fact that the loss is independent of $\Sigma$ in this regime can be directly seen from the optimization.}. Moreover, it admits a closed-form solution given by:
\begin{align}
    V^{\star}=q^{\star}=\frac{1}{\alpha-1}, && \hat{V}^{\star} = \hat{q}^{\star} = \alpha-1
\end{align}
Therefore, the full training loss is given by:
\begin{align}
    \lim\limits_{\lambda\to 0^{+}}\mathcal{E}_{\ell}^{\text{gcm}}(\alpha, \lambda) = 
    \begin{cases} 0  & \text{ for } \alpha \leq 1\\ \frac{1}{2}\left(1-\frac{1}{\alpha}\right) & \text{ for } \alpha > 1
    \end{cases}
\end{align}

\newpage
\section{From Gaussian mixture to single Gaussian}
\label{proof2}

\subsection{Mixture of Gaussians with zero means}
\label{sec:app:zeromean}

We first prove Lemma \ref{lemma:mean_universality} in the main text. First, by Theorem \ref{thm:app:asymp_error}, the asymptotic loss $\mathcal{E}_{\ell}^{\text{gmm}}(\bm \rho, \bm M, \bm \Sigma^\otimes)$ (resp. $\mathcal{E}_{\ell}^{\text{gmm}}(\bm \rho, \bm 0, \bm \Sigma^\otimes)$)  is a deterministic function of $(m_c^\star, q_c^\star, V_c^\star)_{c\in \cC}$, which are the \emph{unique} fixed points of \eqref{eq:app:sp:gmm} (resp \eqref{eq:app:sp:gmm_zeromean}). Since both saddle points equations differ only by setting $m_c = \hat m_c = 0$, Lemma \ref{lemma:mean_universality} is a consequence of the following:
\begin{lemma}
Let $(V_c^\star, q_c^\star)_{c\in\cC}$ be the solutions of Eqs. \eqref{eq:app:sp:gmm_zeromean}. Then, $(0, V_c^\star, q_c^\star)_{c \in \cC}$ satisfy the general fixed point equations of \eqref{eq:app:sp:gmm}.
\end{lemma}

\begin{proof}
If we plug in $m_c = \hat m_c = 0$ for all $c\in \cC$, the equations for $V_c, \hat V_c, q_c, \hat q_c$ become identical in \eqref{eq:app:sp:gmm} and \eqref{eq:app:sp:gmm_zeromean}. It is also easy to check that $\hat m_c = 0$ for all $c$ implies that $m_c = 0$; what remains is to show that the last equation holds, i.e. 
\begin{equation}
    \frac{\alpha}{2} \rho_{c}\sum\limits_{y\in\{-1,+1\}}\mathbb{E}_{\xi\sim\mathcal{N}(0,1)}\left[f_{\ell}(y, \sqrt{q_{c}^\star}\xi, V_{c}^\star)\right] = 0.
\end{equation}

Define the function
\[ g(\omega, V) = f_{\ell}(-1,\omega,V) + f_{\ell}(+1,\omega,V), \]
so that
\[ \hat m_c^\star \propto \dE_{\xi\sim\cN(0, 1)}\left[
g(\sqrt{q_c^\star}\xi, V_c^\star)\right]\]
We shall show that $g$ is odd in $\omega$; since $\xi$ is centered, the lemma will be proven. To do so, we shall show that
\[ f_\ell(y, \omega, V) = -f_\ell(-y, -\omega, V), \]
for all $y \in \{-1, +1\}$, $\omega\in \dR$, and $V\in \dR$. By definition, we have
\[ f_{\ell}(y,\omega,V) = V^{-1}\left(\prox_{V\ell(\cdot, y)}(\omega)-\omega\right), \]
and the linear term in $\omega$ is immediate. For the proximal operator, we use the symmetry of $\ell$ and write
\begin{align*}
   \prox_{V\ell(\cdot, y)}(\omega) &= \argmin_{z\in\cC_1} \left[\frac{1}{2\tau}(z-\omega)^2+\ell(z,y)\right] \\
   &= \argmin_{z\in\cC_1} \left[\frac{1}{2\tau}((-z)-(-\omega))^2+\ell(-z,-y)\right] \\
   &= -\prox_{V\ell(\cdot, -y)}(-\omega),
\end{align*}
which concludes the proof.
\end{proof}

\subsection{Strong universality of ordinary least-squares}
\label{sec:app:uni}
We now have all elements we need to establish the universality of the ordinary least-squares estimator stated in Theorem \ref{thm:ols} in the main.  Our starting point is the ordinary least-squares problem for the Gaussian Mixture Model in the overdetermined regime $\alpha > 1$. In this cave, the training loss is given by eq.~\eqref{eq:app:loss:ridge:gmm} with $(V^{\star}_{c},q^{\star}_{c})_{c\in\mathcal{C}}$ unique solutions of the following equations:
\begin{align}
&\begin{cases}
\hat{V}_{c} = \frac{\alpha \rho_{c}}{1+V_{c}}\\
\hat{q}_{c} = \alpha \rho_{c}\frac{1+q_{c}}{(1+V_{c})^2}\\
\end{cases}, &&
\begin{cases}
	V_{c} &= \frac{1}{d}\tr\Sigma_{c}\left(\sum\limits_{c'\in\mathcal{C}}\hat{V}_{c'}\Sigma_{c'}\right)^{-1}\\
q_{c} &= \frac{1}{d}\sum\limits_{c'\in\mathcal{C}}\left[\hat{q}_{c'}\tr\Sigma_{c'}\Sigma_{c}\left(\sum\limits_{c''\in\mathcal{C}}\hat{V}_{c''}\Sigma_{c''}\right)^{-2}\right]
\end{cases}
\label{ref:flo0}
\end{align}	

We shall now show how to reduce these equation to a simple analytical formula, equivalent to the one of a single Gaussian. Combining the equations for $\hat{V}_{c}$ and $V_{c}$, one sees that the fixed point must satisfy the following identity:
\begin{align}
    \sum\limits_{c\in\mathcal{C}}\hat{V}^\star_{c}V^\star_{c} = 1
\end{align}
Similarly, multiplying the equation for $q_{c}$ by $\hat{V}_{c}$, summing over $c\in\mathcal{C}$ and doing the same for the equation for $\hat{q}_{c}$ with $V_{c}$, we get a second identity satisfied by the fixed-point:
\begin{align}
    \sum\limits_{c\in\mathcal{C}}\left(\hat{V}_{c}^\star q_{c}^\star-V_{c}^\star\hat{q}_{c}^\star\right) = 0
\end{align}
Note that, at this point these relations could have been derived for any loss functions. For the specific case of the square loss, further substituting the hat variables, these conditions are equivalent to:
\begin{align}
    \sum\limits_{c\in\mathcal{C}}\rho_{c}\frac{V_{c}^\star}{1+V_{c}^\star} = \frac{1}{\alpha} \label{ref:flo1} \\ \sum\limits_{c\in\mathcal{C}}\rho_{c}\frac{V_{c}-q_{c}}{(1+V_{c})^{2}} = 0 \label{ref:flo2}
\end{align}
We thus find, combining the eq.~\eqref{ref:flo0} for $\hat q_c$ with eq.~\eqref{ref:flo2}
\begin{align}
    \sum\limits_{c\in\mathcal{C}} \hat q_c^\star = \sum\limits_{c\in\mathcal{C}} \alpha \rho_c\frac{1+V^\star_{c}}{(1+V^\star_{c})^2} 
    = \sum\limits_{c\in\mathcal{C}} \alpha \rho_c\frac 1{1+V^\star_{c}} 
    \label{ref:flo3} 
\end{align}
Our goal is to evaluate the loss at the fixed point, which is given by eq.~\eqref{eq:app:loss:ridge:gmm-flo}:
\begin{align}
    \mathcal{E}^{\text{gmm}}_{\ell} = \sum\limits_{c\in\mathcal{C}} \frac{\hat q^\star_c}{2 \alpha}
\end{align}
Combining this definition with eqs.~\eqref{ref:flo2} and  \eqref{ref:flo3}, we find that
\begin{align}
    2\mathcal{E}^{\text{gmm}}_{\ell} + \frac 1{\alpha} = 
    \sum\limits_{c\in\mathcal{C}}  \rho_c\frac 1{1+V^\star_{c}} +  \sum\limits_{c\in\mathcal{C}}\rho_c\frac{V^\star_{c}}{1+V^\star_{c}} = 1
\end{align}
so that finally, we reach the promised result:
\begin{align}
    \lim\limits_{\lambda\to 0^{+}} \mathcal{E}^{\text{gmm}}_{\ell}(\alpha,\lambda,K) = \frac{1}{2}\left(1-\frac{1}{\alpha}\right)_{+} = \lim\limits_{\lambda\to 0^{+}}\mathcal{E}_{\ell}^{\text{gcm}}(\alpha,\lambda)
\end{align}
as claimed in Theorem \ref{thm:ols} in the main.
\newpage
\section{Numerical Simulations}
\label{numerical_simulations}

In this section, we provide further details concerning the protocol we used to perform the numerical simulations, which corroborate the theoretical results exemplified in the main manuscript. All codes are publicly available on the GitHub repository associated to the current paper at  
\href{https://github.com/IdePHICS/RandomLabelsUniversality}{https://github.com/IdePHICS/RandomLabelsUniversality}. For concreteness and completeness, we illustrate these simulations with yet again a different case with respect to what was already presented in the main text in Fig. \ref{fig:tiny_imagenet_finite_reg}, where we use a smaller version of the well-known Imagenet benchmark \cite{deng2009imagenet}. It is made of $100.000$ natural images, downsampled to $64 \times 64$ pixels each and grouped into $200$ different classes.
 
\begin{figure}[ht]
\begin{center}
\centerline{\includegraphics[width=500pt]{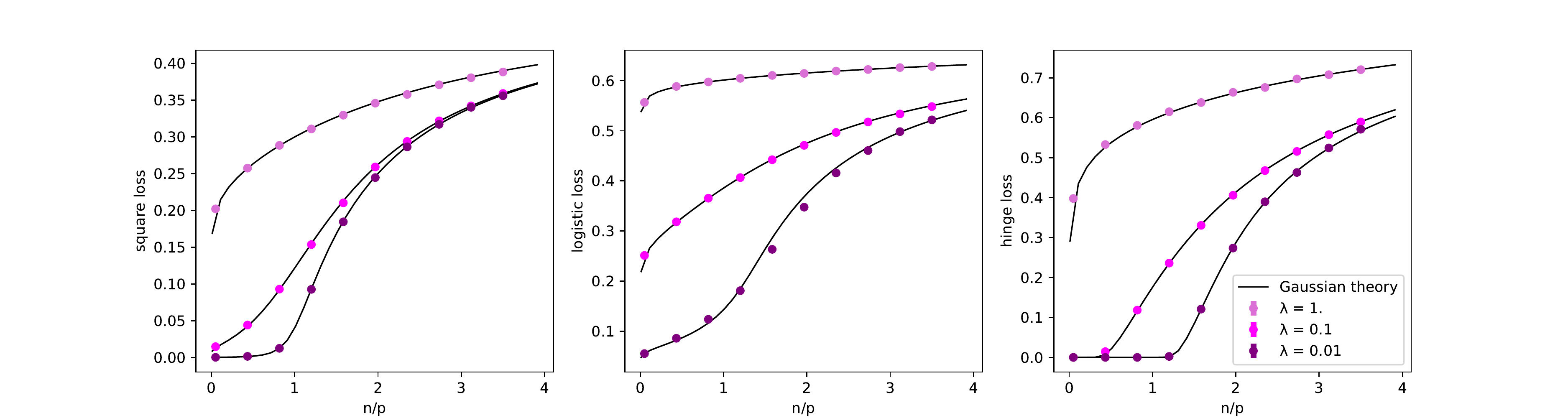}}
\caption{Numerical simulations of universality: As in Fig. \ref{fig:all_datasets_finite_reg}, this figure shows the training loss as a function of the number of samples $n$ per dimension $p$ at various values of $\lambda$ for another data set we used here for completeness. Here we used a grayscale tiny-Imagenet pre-processed with Gaussian random features and $\mbox{tanh}$ non-linearity. In the left panel the square loss, in the middle panel the logistic loss and in the right panel the hinge loss. The coloured dots refer to numerical simulations while the black solid lines correspond to the theoretical prediction of single Gaussian with corresponding input covariance matrices. The numerical simulations are averaged over $10$ different realizations.}
\label{fig:tiny_imagenet_finite_reg}
\end{center}
\vskip -0.2in
\end{figure}

In all the numerical experiments on real datasets shown so far, we have both normalized and then pre-processed the datasets with either random features or wavelet-scattering transform. Fig. \ref{fig:all_datasets_finite_reg_direct} compares instead the predictions of the Gaussian theory with respect to the numerical simulations on MNIST, fashion-MNIST, CIFAR10 and tiny ImageNet when no pre-processing is applied. As can be seen, despite the overall quite good agreement between theory and numerical experiments, we start observing some (very) small deviations from the Gaussian predictions. Indeed, as it is shown in sec. \ref{evaluation_data_covariance}, the covariance matrices associated to the different modes of the underlying real data distribution are, in this case, more heterogeneous than the ones observed when a pre-processing stage is applied. This is consistent with the homogeneous assumption in Theorem \ref{thm:Gauss_universality} and implying Gaussian Universality.

\begin{figure}[t!]
\begin{center}
\centerline{\includegraphics[width=500pt]{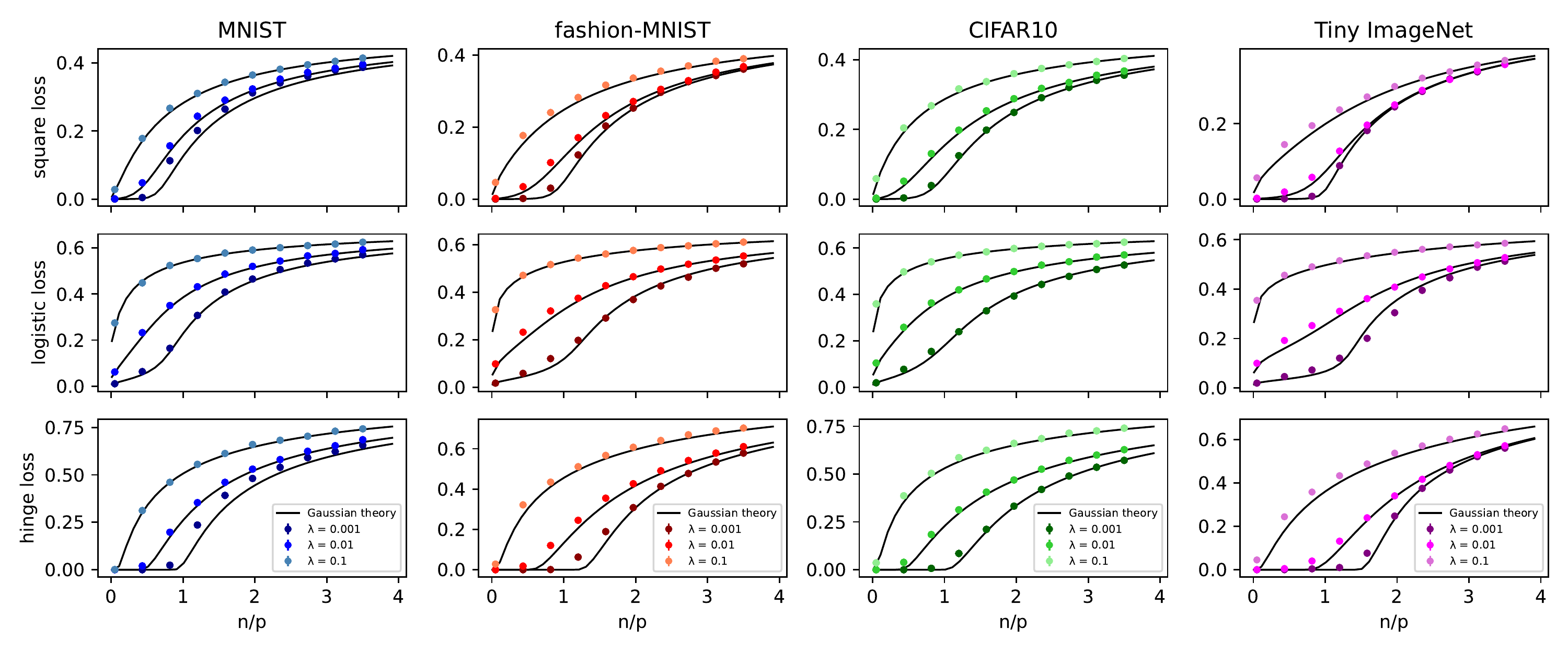}}
\caption{This figure shows the training loss as a function of the number of samples $n$ per dimension $p$ at finite regularization $\lambda$. In the top panel the square loss, and in the bottom panel the hinge loss. The first column refers to MNIST, the second column corresponds to fashion-MNIST, the third column corresponds to CIFAR10 in grayscale, the fourth column corresponds to tiny ImageNet in grayscale. Black solid lines correspond to the outcome of the replica calculation, obtained by assigning to $\Sigma$ the covariance matrix of each dataset. The coloured dots correspond to the simulations for different values of $\lambda$, as specified in the plot legend. Simulations are averaged over $10$ samples \& the error bars are not visible at the plot scale.}
\label{fig:all_datasets_finite_reg_direct}
\end{center}
\vskip -0.4in
\end{figure}

\paragraph{Dataset generation.} As we have seen in the main manuscript, we basically deal with three different types of datasets. Two of them are synthetic datasets and correspond to i.i.d Gaussian input data-points and Gaussian Mixtures. The remaining one accounts for real datasets, such as MNIST \cite{deng2012mnist}, fashion-MNIST \cite{xiao2017/online} and CIFAR10 \cite{xiao2017/online} in grayscale, pre-processed with either random feature maps \cite{rahimi2007random} or through wavelet scattering transform \cite{bruna2013invariant}. The procedure used to generate these kinds of datasets is exemplified in sec. \ref{sec:numerics}. For the sake of clarity, we summarized it through the pseudo code in algorithm \ref{alg:app:dataset}. 

\begin{algorithm}[H]
   \caption{Generating dataset $\mathcal{D} = \{\bm{x}^{\mu}, y^{\mu}\}_{\mu=1}^{n}$}
   \label{alg:app:dataset}
\begin{algorithmic}
   \STATE {\bfseries Input:} Integer $p$, flag \emph{dataset}, matrix $F \in \mathbb{R}^{d\times p}$ of random Gaussian features
   \STATE {\bfseries If} the \emph{dataset type} is i.i.d. Gaussian:
        \STATE \hspace{2mm} Sample each input data-point as $\bm{x}^{\mu} \sim \mathcal{N}\left(0, \bm{I} \right)$, with $\bm{I} \in \mathbb{R}^{p\times p}$ the identity matrix;
   \STATE {\bfseries Else if} the \emph{dataset type} is a Gaussian Mixture:
        \STATE \hspace{2mm} Sample each input data-point as $\bm{x}^{\mu} \sim \sum_{k=1}^K \rho_k \ \mathcal{N}\left(\bm{\mu}_k, \Sigma_k \right)$, with $\bm{\mu}_{k}$ being the centroid of
        \STATE \hspace{2mm} the k-th cluster and $\Sigma_k$ the corresponding covariance matrix;
    \STATE {\bfseries Else if} the \emph{dataset type} is a real dataset pre-processed with random gaussian features:
        \STATE \hspace{2mm} Load the real dataset samples $\bm{z}^{\mu} \ \forall \mu = 1,...,n$ with Pytorch dataloaders;
        \STATE \hspace{2mm} Assign $\bm{x}^{\mu} \rightarrow \sigma \left( \bm{z}^{\mu} F\right)$;
    \STATE {\bfseries Else if} the \emph{dataset type} is a real dataset pre-processed with wavelet scattering:
        \STATE \hspace{2mm} Load the real dataset samples $\bm{z}^{\mu} \ \forall \mu = 1,...,n$;
        \STATE \hspace{2mm} Apply wavelet scattering transform on $\bm{z}^{\mu}$;
   \STATE Sample the labels according to the Rademacher distribution as $y^{\mu} \sim \frac{1}{2}\left(\delta_{+1} + \delta_{-1} \right)$
   \STATE {\bfseries Return: $\mathcal{D} = \{\bm{x}^{\mu}, y^{\mu}\}_{\mu=1}^{n}$} 
\end{algorithmic}
\end{algorithm}

The real datasets are loaded through Pytorch dataloaders  \cite{NEURIPS2019_9015}. In particular, the dataloader of CIFAR10 includes a grayscale transformation of the dataset in order to reduce to one the number of input channels of the RGB colour encoding scheme. The wavelet scattering transform is instead implemented by means of the Kymatio Python library \cite{andreux2020kymatio}. Note that, with the purpose of speeding up the realization of the learning curves and to reduce fluctuations, the pre-processed real datasets are generated once for all through algorithm \ref{alg:app:dataset} and then stored in a hdf5 file.  

\paragraph{Learning phase.} Given the dataset generated as in algorithm \ref{alg:app:dataset}, the aim is to infer the estimator $\bm{\theta}$ minimizing the empirical risk as in eq. \eqref{eq:def_min_problem_main} of the main manuscript. In the present work we consider three distinct kinds of loss functions:
\begin{enumerate}[wide=2pt]
    \item \textbf{Square Loss.} In this specific case, the goal is to solve the following optimization problem:
    \begin{equation}
    \label{eq:app:def_min_problem_square_loss}
    \widehat\cR_n^*(\bm X, \bm y) = \inf_{\bm \theta \in \cS_p} \frac{1}{2n}\sum_{\mu=1}^n (\bm \theta^\top \bm x_\mu -  y_\mu)^2 + 
    \frac{\lambda}{2} \vert\vert \bm{\bm \theta} \vert\vert_2^2 ,
    \end{equation}
    The estimator can be here determined through the Moore-Penrose inverse as it follows, without relying on any learning algorithm:
    \begin{equation}
    \bm \theta = \begin{cases}
    \left( \bm{X}^{\top}\bm{X} + \lambda \bm{I}_{p} \right)^{-1} \bm{X}^{\top} \bm{y} , & \text{if} \ n > p\\\\
    \bm{X}^{\top}\left(\bm{X}\bm{X}^{\top} + \lambda \bm{I}_{n} \right)^{-1} \bm{y} , & \text{if} \ p > n 
    \end{cases}
    \end{equation}
    \item \textbf{Logistic Loss.} In this specific case, the goal is to solve the following optimization problem:
    \begin{equation}
    \label{eq:app:def_min_problem_logistic_loss}
    \widehat\cR_n^*(\bm X, \bm y) = \inf_{\bm \theta \in \cS_p} \frac{1}{n}\sum_{\mu=1}^n \mbox{log}\left(1 + \mbox{exp}(-y_\mu\bm \theta^\top \bm x_\mu)\right) + 
    \frac{\lambda}{2} \vert\vert \bm{\bm \theta} \vert\vert_2^2 ,
    \end{equation}
    Since the estimator of logistic regression can not be determined through an explicit closed formula, we here made use of the \emph{lbgfs} solver with \emph{penalty} set to $\ell_2$. This optimizer corresponds to a Gradient Descent (GD)-like second order optimization method and it is implemented in the LogisticRegression class of the Scikit-Learn Python library \cite{scikit-learn}. The GD algorithm stops either if a maximum number of iterations has been reached or if the maximum component of the gradient goes below a certain threshold. We fixed this tolerance to $1e-5$ and the maximum number of iterations to $1e4$. 
    \item{\textbf{Hinge Loss.}} In this specific case, the goal is to solve the following optimization problem:
    \begin{equation}
    \label{eq:app:def_min_problem_hinge_loss}
    \widehat\cR_n^*(\bm X, \bm y) = \inf_{\bm \theta \in \cS_p} \frac{1}{n}\sum_{\mu=1}^n \mbox{max}\left(0, 1 - y_\mu\bm \theta^\top \bm x_\mu)\right) + 
    \frac{\lambda}{2} \vert\vert \bm{\bm \theta} \vert\vert_2^2 ,
    \end{equation}
    As for logistic regression, even in this case we can not rely on any explicit formula for the estimator, it has rather to be inferred by means of learning algorithm. In particular, for the simulations at finite regularization strength, we made use of the LinearSVC class provided by Scikit-Learn \cite{scikit-learn} and implementing the Support Vector Classification (SVC) with linear kernels and $L_2$ regularization if \emph{penalty} is set to $\ell_2$. In this case, we set the tolerance of convergence to $1e-5$ and the maximum number of iterations to $5e5$. Unfortunately, LinearSVC struggles to converge for vanishing regularization strengths. Therefore, we made use of CVXPY \cite{agrawal2018rewriting,diamond2016cvxpy} in order to perform the simulations at $\lambda = 1e-15$. CVXPY is an open-source Python-embedded modeling language for convex optimization problems. We set the \emph{solver} option to None, in this way CVXPY chose automatically the most specialized solver for the optimization problem type. While being slower than LinearSVC, CVXPY guarantees convergence at vanishing regularization strengths. 
\end{enumerate}

At the end of the training process, we evaluate the training loss $\ell$ on the minimizer of the corresponding empirical risk minimization problem. To get the learning curves, we then repeat the whole process for a specified range of $n/p$ and for a certain number of different realization of the learning problem, as exemplified in algorithm \ref{alg:app:learning}.

\begin{algorithm}[H]
   \caption{Learning curve}
   \label{alg:app:learning}
\begin{algorithmic}
   \STATE {\bfseries Input:} range of $n/p$, flag \emph{dataset type}, flag \emph{which estimator}
   \STATE {\bfseries For} $seed$ in a specified number of seeds {\bfseries do}:
   \STATE \hspace{2mm} {\bfseries For} $n/p$ in a specified range {\bfseries do}:
   \STATE \hspace{6mm} Choose the dataset according to \emph{dataset type};
   \STATE \hspace{6mm} Compute the estimator according to the desired optimization problem as in (i)-(iii);
   \STATE \hspace{6mm} Compute the training loss $\ell$ at fixed $n/p$;
   \STATE \hspace{2mm} Update the mean train loss and its standard deviation with the new contribution from the current seed.
   \STATE {\bfseries Return:} Mean train loss and standard deviation as a function of $n/p$.
\end{algorithmic}
\end{algorithm}
\section{Empirical evidence of the homogenity assumption}
\label{evaluation_data_covariance}

As seen in the counter example illustrated in Fig.~\ref{fig:no_match_reg_only_three_blocks}, in the case of very heterogeneous Gaussian Mixtures we can observe small deviations from universality both at zero and finite regularization. However, this disagreement between single Gaussian and Gaussian Mixtures does not appear in the experiments with real datasets of Fig.~\ref{fig:all_datasets_zero_reg} and Fig.~\ref{fig:all_datasets_finite_reg}, despite their certainly multi-modal and mode-heterogeneity nature. First, we must acknowledge that deviations are, in general, observed to be small with respect to the homogeneous case, and that the data presented in Fig.~\ref{fig:no_match_reg_only_three_blocks} were carefully tuned so that the difference is visible. 

Additionally, in this section, we also empirically demonstrate the similarity among the empirical correlation matrices of the various modes characterizing real dataset distributions.  Fig.~\ref{fig:true_cifar10_cov} shows the correlation matrix of all grayscale CIFAR-10 images depicting airplanes (leftmost), automobiles (middle) and trucks (rightmost) respectively. The point we wish to convey
in this plot is that, despite the fact that there exists some modes of the CIFAR-10 empirical distribution which display a consistently different correlation structure (airplane mode) with respect to the other modes (automobile and truck mode), there exists some others which look like more similar among each other (automobile and truck mode).  

\begin{figure}[ht]
\begin{center}
\centerline{\includegraphics[width=380pt]{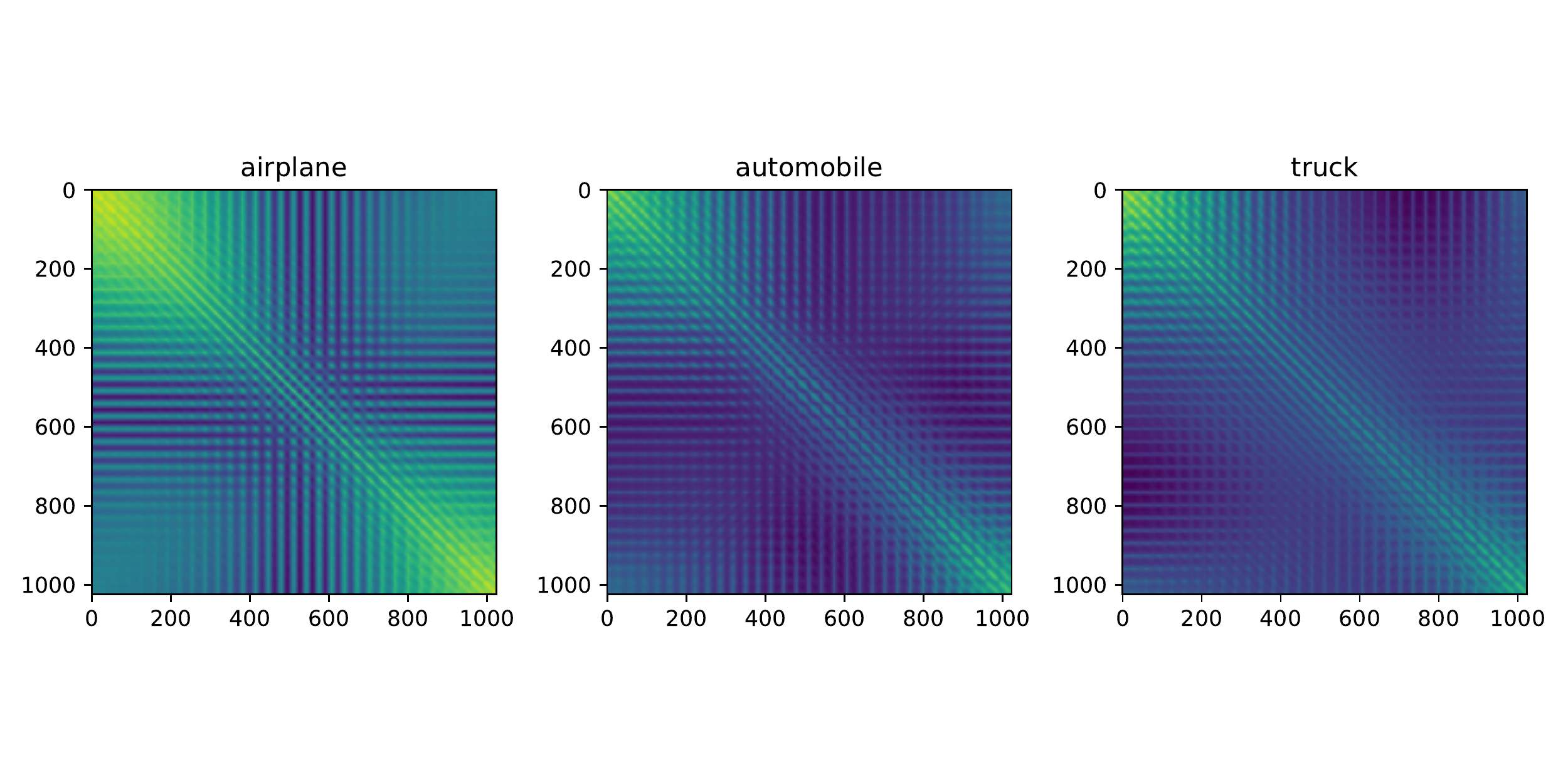}}
\vspace{-10mm}
\caption{Input data correlation matrix of grayscale CIFAR10, conditioned on the true labels, e.g. airplane (leftmost), automobile (middle), truck (rightmost). Lighter colors refer to stronger correlation. }
\label{fig:true_cifar10_cov}
\end{center}
\vskip -0.2in
\end{figure}

As can be seen in Fig.~\ref{fig:rf_cifar10_cov} and Fig.~\ref{fig:ws_cifar10_cov}, the structure similarity of the covariance matrices of the various mode is further enhanced when pre-processing grayscale CIFAR-10 with both Gaussian random feature maps and wavelet scattering transforms, at the point that even the less similar modes in the raw dataset conform to the others (see airplane mode). 

\begin{figure}[ht]
\begin{center}
\centerline{\includegraphics[width=380pt]{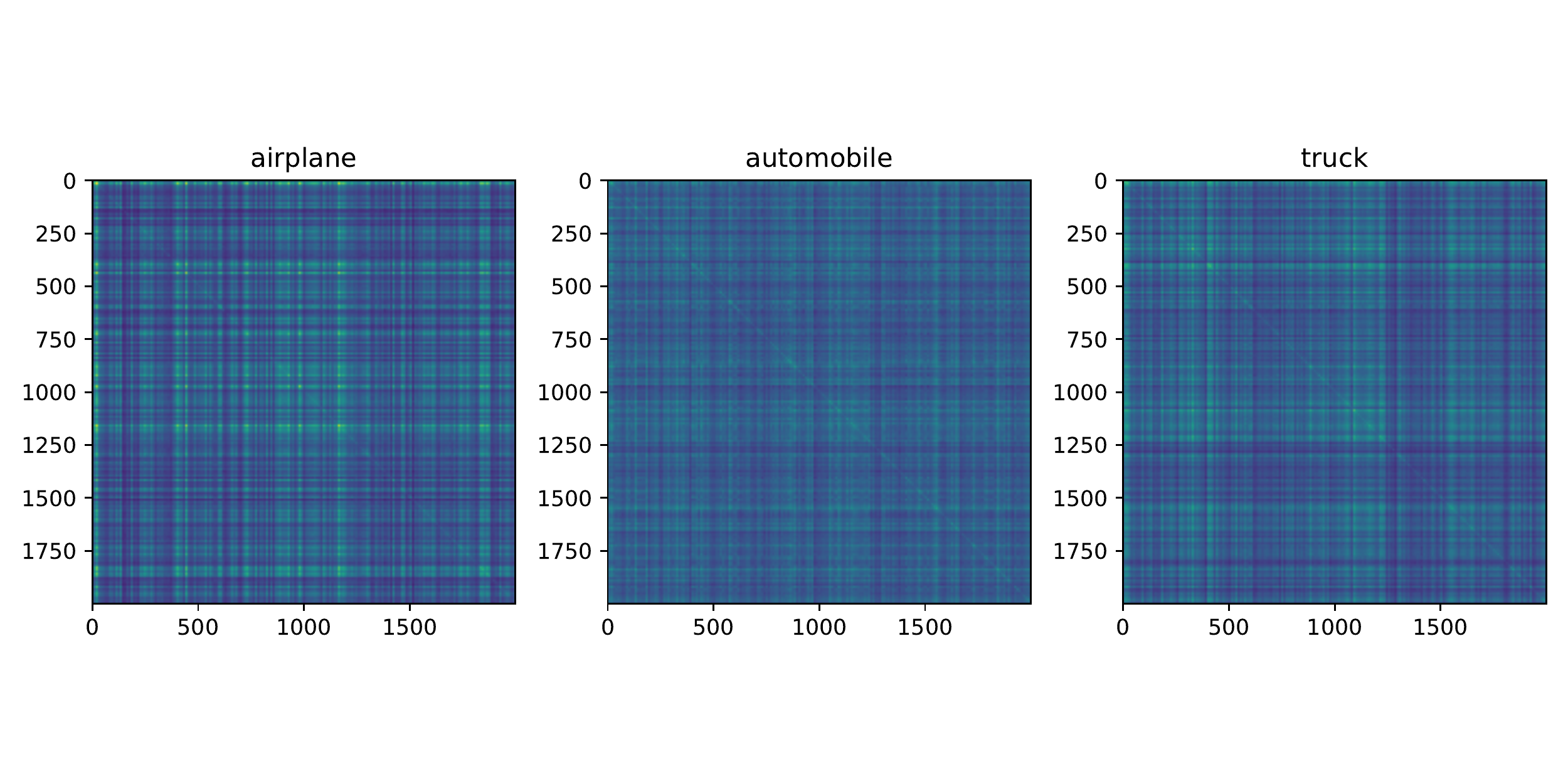}}
\vspace{-10mm}
\caption{Input data correlation matrix of grayscale CIFAR10 pre-processed with Gaussian random features and $\mbox{erf}$ non-linearity. The correlation matrices are conditioned on the true labels, e.g. airplane (leftmost), automobile (middle), truck (rightmost). Lighter colors refer to stronger correlation. }
\label{fig:rf_cifar10_cov}
\end{center}
\vskip -0.2in
\end{figure}

\begin{figure}[ht]
\begin{center}
\centerline{\includegraphics[width=380pt]{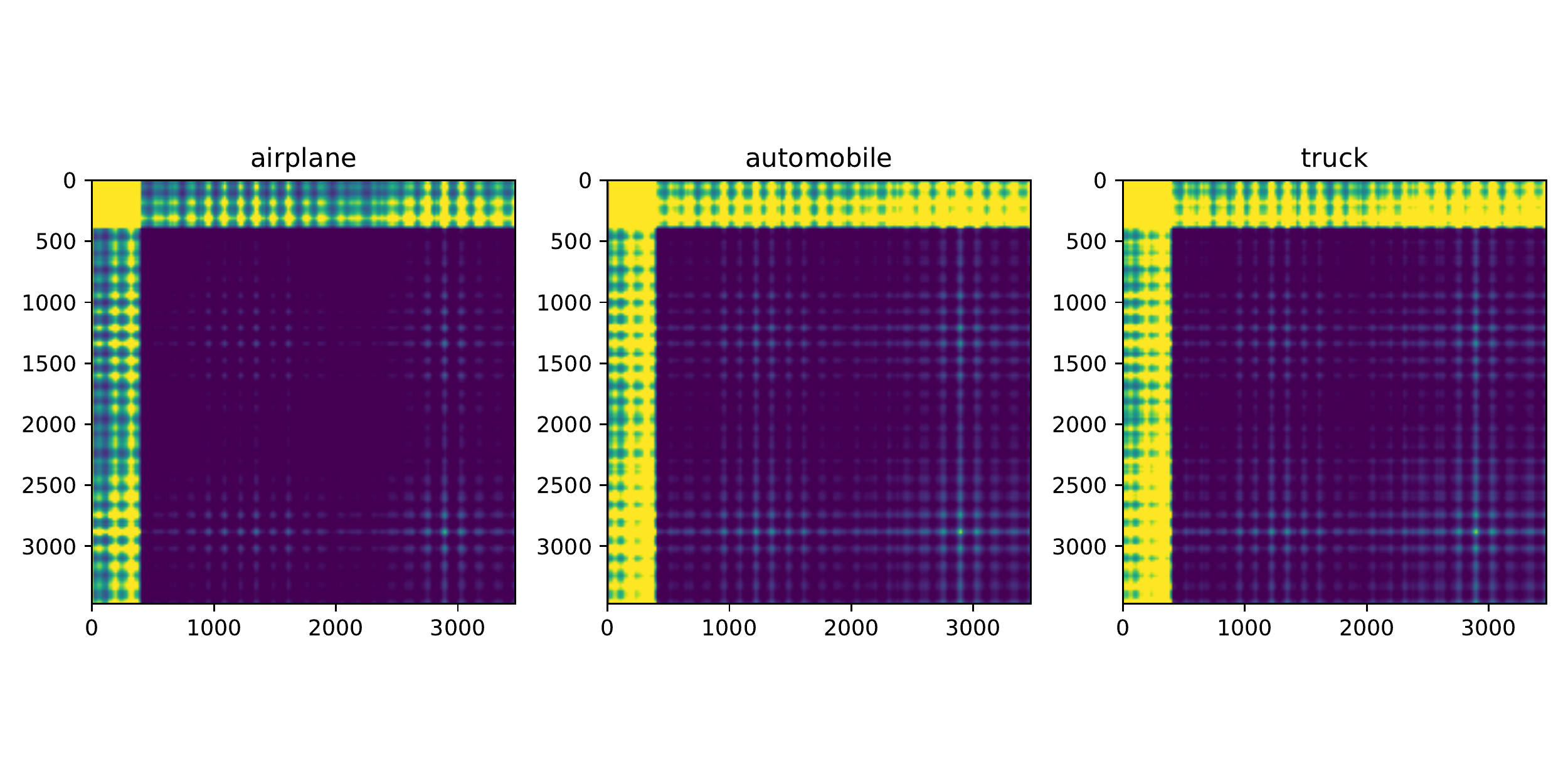}}
\vspace{-10mm}
\caption{Input data correlation matrix of grayscale CIFAR10 pre-processed with wavelet scattering transform. The correlation matrices are conditioned on the true labels, e.g. airplane (leftmost), automobile (middle), truck (rightmost). Lighter colors refer to stronger correlation.}
\label{fig:ws_cifar10_cov}
\end{center}
\vskip -0.2in
\end{figure}

\end{document}